\documentclass{article}
\usepackage{arxiv}

\usepackage[utf8]{inputenc} 
\usepackage[T1]{fontenc}    
\usepackage{hyperref}       
\usepackage{url}            
\usepackage{booktabs}       
\usepackage{amsfonts}       
\usepackage{nicefrac}       
\usepackage{microtype}      
\usepackage{lipsum}		
\usepackage{graphicx}
\usepackage{natbib}
\usepackage{doi}

\usepackage{mathtools}

\def\1{\bm{1}}

\DeclareMathAlphabet{\mathsfit}{\encodingdefault}{\sfdefault}{m}{sl}
\SetMathAlphabet{\mathsfit}{bold}{\encodingdefault}{\sfdefault}{bx}{n}

\def\gA{{\mathcal{A}}}

\def\gS{{\mathcal{S}}}
\def\gT{{\mathcal{T}}}

\def\0{{\bf 0}}
\def\1{{\bf 1}}

\def\AM{{\mathcal A}}

\def\PM{{\mathcal P}}
\def\SM{{\mathcal S}}

\def\RB{{\mathbb R}}
\def\EB{{\mathbb E}}

\def\PB{{\mathbb P}}

\def\ie{{\em i.e.\/}}

\usepackage{amsmath}
\usepackage{amssymb}
\usepackage{mathtools}
\usepackage{amsthm}
\usepackage{ dsfont }
\usepackage{diagbox}

\theoremstyle{plain}
\newtheorem{thm}{Theorem}[section]
\newtheorem{lem}{Lemma}[section]

\newtheorem{asmp}{Assumption}[section]

\usepackage{algorithm}
\usepackage{algorithmic}
\usepackage{bm}

\title{Federated Reinforcement Learning with Constraint Heterogeneity}

\author{ 
	\textbf{Hao Jin}\\
	jin.hao@pku.edu.cn\\
	Peking University\\
	\And
	\textbf{Liangyu Zhang}\\
	zhangliangyu@pku.edu.cn\\
	Peking University\\
    \And
	\textbf{Zhihua Zhang}\\
	zhzhang@math.pku.edu.cn\\
	Peking University\\
}

\date{}

\renewcommand{\shorttitle}{FedRL with Constraint Heterogeneity}

\begin{document}
\maketitle

\begin{abstract}
We study a Federated Reinforcement Learning (FedRL) problem with constraint heterogeneity.
In our setting, we aim to solve a reinforcement learning problem with multiple constraints while $N$ training agents are located in $N$ different environments with limited access to the constraint signals and they are expected to collaboratively learn a policy satisfying all constraint signals.
Such learning problems are prevalent in scenarios of Large Language Model (LLM) fine-tuning and healthcare applications.
To solve the problem, we propose federated primal-dual policy optimization methods based on traditional policy gradient methods.
Specifically, we introduce $N$ local Lagrange functions for agents to perform local policy updates, and these agents are then scheduled to periodically communicate on their local policies. 
Taking natural policy gradient (NPG) and proximal policy optimization (PPO) as policy optimization methods, we mainly focus on two instances of our algorithms, \ie, {FedNPG} and {FedPPO}.
We show that
FedNPG achieves global convergence with an $\tilde{O}(1/\sqrt{T})$ rate, and FedPPO efficiently solves complicated learning tasks with the use of deep neural networks.
\end{abstract}


\section{Introduction}
Recent years have witnessed the growing popularity of reinforcement learning (RL) \citep{sutton1998introduction} in solving challenging problems, such as playing the games of Go \citep{hessel2018rainbow,silver2016mastering,silver2017mastering} and driving automobiles \citep{kiran2021deep,fayjie2018driverless,chen2019model}.
In classical settings of RL, the agent continually interacts with a fixed environment, and utilizes such collected experience to learn a policy maximizing specified reward signals.
However, real-life applications have raised many new practical settings for policy learning \citep{qin2021density,jin2022federated,junges2016safety,chow2017risk,liu2019lifelong}.
As one of these emerging settings, federated reinforcement learning (FedRL) focuses on how to coordinate agents located separately to learn a well-performing policy without privacy violations on individually collected experience~\citep{liu2019lifelong,zhuo2019federated, wang2020federated,nadiger2019federated,jin2022federated}.
In the setting of FedRL, major challenges of policy learning come from the misalignment of involved agents, which is typically known as heterogeneity \citep{jin2022federated,nadiger2019federated}.

In this paper, we mainly consider the \textit{constraint heterogeneity}, which is prevalent due to the trend of distributed data collection and the necessity of privacy preservation.
Suppose there is a reinforcement learning problem with multiple constraints, and the monitoring on a constraint signal is sometimes costly.
During the training process, it is impossible for any single agent to collect the whole set of constraint signals.
With the introduction of a federated platform, any single participated agent only has to focus on a certain constraint signal while it is finally guaranteed with a well-behaved model satisfying all constraints.
\textit{Constraint heterogeneity} naturally arises in the training process, where different agents have access to different constraint signals.
Take the fine-tuning of Large Language Models as an example (LLM).
There is a common science that LLM trained on raw Internet-based data suffers from problems known as "social bias" \cite{gallegos2023bias}, and introducing constraints on fairness of generated text is believed to address the problem.
Accompanied with the increasing trend of federated optimization in training LLMs \cite{ro2022scaling,chen2023federated}, it is hard to guarantee an alignment of constraint signals received in different devices corresponding to different groups of users.
To learn the LLM satisfying constraints of all participated devices, it is natural to consider constraint heterogeneity in federated reinforcement learning.
In fact, constraint heterogeneity is also prevalent in applications where training data is distributed among different agents and the labelling on constraint signals is costly, such as derivation of dynamic treatment regime (DTR) for patients in different physical conditions \citep{liu2016robust,zhang2019reinforcement}. 

We formulate FedRL with constraint heterogeneity as a seven-tuple of $\langle \SM, \AM, r, \{(c_i,d_i)\}_{i=1}^N, \gamma, \PM, \{\Gamma_i\}_{i=1}^N\rangle$.
In our setting,
$N$ agents share the same state space $\SM$, action space $\AM$, reward function $r$, discounted factor $\gamma$, and transition dynamic $\PM$; the $i$-th constraint is modeled with its associated cost functions $c_i$ and corresponding threshold $d_i$;$\{\Gamma_i\}_{i=1}^N$ is the source of constraint heterogeneity, and $\Gamma_i$ indexes the constraints accessible for the $i$-th agent.
Without loss of generality, we assume in our discussion that the $i$-th agent can only access the $i$-th constraint $(c_i,d_i)$ in addition to the reward function $r$, \ie, $~\Gamma_i=\{i\}$.
In this way, none of the $N$ agents has full access to all of $N$ constraints, which prevents any agent from locally adopting existing methods of constrained reinforcement learning \citep{chow2017risk,liang2018accelerated,tessler2018reward,bohez2019value,xu2021crpo,liu2022constrained} to solve the problem.

We propose a class of federated primal-dual policy optimization methods to solve FedRL problems with constraint heterogeneity.
Suppose that we parameterize a policy $\pi$ with $\theta\in\Theta$, and denote its cumulative performance w.r.t.\ reward functions $r$ and cost functions $\{c_i\}_{i=1}^N$ as functions of $\theta$, \ie, $~J_r(\theta)$ and $\{J_{c_i}(\theta)\}_{i=1}^N$.
The policy learning is then transformed into the following constrained optimization problem w.r.t. $\theta$:
\begin{align*}
    \max_\theta~~& J_r(\theta),\\
    \texttt{s.t.}~~ J_{c_i}(\theta)&\leq d_i,~ i\in [N] := \{1,2,\dots,N\}.
\end{align*}
When there is no constraint heterogeneity, primal-dual methods considers the following Lagrange function:
\begin{equation*}
L_0(\lambda,\theta)=J_r(\theta)+\sum_{i=1}^N \lambda_k(d_i-J_{c_i}(\theta)),
\end{equation*}
where $\lambda=(\lambda_1,\dots,\lambda_N)^T\in\mathbb{R}_+^N$ serves as the vector of Lagrange multipliers associated with the $N$ constraint functions.
In this way, these methods turn to solve the min-max optimization based on $L_0(\lambda,\theta)$ through gradient-descent-ascent, which solves original problem when strong duality holds \citep{oh2017zero,paternain2019constrained,ding2020natural}.
However, the existence of constraint heterogeneity prohibits us from using $L_0(\lambda,\theta)$ to update local policies.
To address the issue, we decompose $L_0(\lambda,\theta)$ into $N$ local Lagrange functions $\{L_i(\lambda_i,\theta)\}_{i=1}^N$.
Specifically, the $i$-th local Lagrange function does not require additional knowledge other than $r$ and $(c_i,d_i)$ and the $i$-th agent is able to conduct primal-dual updates accordingly.
Furthermore, our methods require $N$ agents to periodically communicate their local policies in order to find a policy satisfying all the constraints.
Finally, we instantiate different different federated algorithms with different policy optimization methods.
Specifically, in FedRL problems with tabular environments, we devise the {FedNPG} algorithm
based on natural policy gradient. 
We show that {FedNPG} achieves an $\widetilde{O}(1/\sqrt{T})$ global convergence rate, and also conduct empirical analysis. 
To solve complicated FedRL tasks on real data, we resort to the deep neural networks, devise the {FedPPO} algorithm by utilizing proximal policy optimization, and empirically validate its performance  in CartPole, Acrobot and Inverted-Pendulum.

In summary, our paper mainly offers the following contributions:
\begin{itemize}
    \item We are the first to consider federated reinforcement learning (FedRL) with constraint heterogeneity, in which different agents have access to different constraints and the learning goal is to find an optimal policy satisfying all constraints.
    \item We propose a class of federated primal-dual policy optimization methods to solve FedRL problems with constraint heterogeneity, which involves the introduction of local Lagrange functions and periodic aggregation of locally updated policies.
    \item We analyze theoretical performance of our method when adopting NPG as the policy optimizer, and prove that {FedNPG} achieves global convergence at an $\tilde{O}(1/\sqrt{T})$ rate.
    Moreover, we evaluate empirical performance of {FedPPO} in complicated tasks of FedRL with the use of deep neural networks.
\end{itemize}

\section{Constrained Reinforcement Learning}
Before formulating our concerned FedRL with constraint heterogeneity, we first discuss the constrained reinforcement learning (constrained RL) problem including constrained Markov decision processes
and primal-dual policy optimization methods.

\subsection{Constrained Markov decision processes}
Constrained reinforcement learning problems are usually modeled with constrained Markov decision processes (CMDPs), which is formulated as the tuple $\langle \SM,\AM,r,\{(c_i,d_i)\}_{i=1}^N,\gamma,\PM\rangle$.
Given any policy $\pi$, the cumulative reward and cumulative cost functions are defined as follows:
\begin{align*}
    J_r(\pi)&=\EB_{s_0\sim\rho}\left [ \sum_{t=0}^\infty\gamma^t r(s_t,a_t)\mathrel{\bigg|} a_t\sim\pi(\cdot|s_t)\right],\\
    J_{c_i}(\pi)&=\EB_{s_0\sim\rho}\left [ \sum_{t=0}^\infty\gamma^t c_i(s_t,a_t)\mathrel{\bigg|},a_t\sim\pi(\cdot|s_t)\right],  
\end{align*}
where $\rho$ indicates the distribution of initial states, and the expectations are taken over the randomness of both the policy $\pi$ and the environment $\PM$.
In this way, the learning goal can be formulated as the following optimization problem:
\begin{equation}
\label{Prob_CMDP}
    \max_\pi J_r(\pi),~~\texttt{s.t.} J_{c_i}(\pi)\leq d_i,~i\in [N].
\end{equation}
We will always assume the problem above is feasible and denote the optimal policy as $\pi^*$.

\subsection{Primal-dual policy optimization methods}
Policy optimization methods are prevalent in solving reinforcement learning problems with large state spaces.
Suppose the policy $\pi$ is parameterized by $\theta\in\Theta\subset\RB^d$ (denoted $~\pi_\theta$).
These methods transform Problem \ref{Prob_CMDP} as the following constrained optimization problem w.r.t. $\theta$:
\begin{equation}
\label{Prob_CMDP_theta}
    \max_{\theta\in\Theta} J_r(\theta),~~\texttt{s.t.} J_{c_i}(\theta)\leq d_i,~i\in [N],
\end{equation}
where $J_r(\theta)$ and $J_{c_i}(\theta)$ represent $J_r(\pi_\theta)$ and $J_{c_i} (\pi_\theta)$.
To solve Problem \ref{Prob_CMDP_theta}, primal-dual methods augment the objective as follows:
\begin{equation}
\label{aug_obj}
L_0(\lambda,\theta)=J_r(\theta)+\sum_{i=1}^N\lambda_i(d_i-J_{c_i}(\theta)),
\end{equation}
where $\lambda_i\geq 0$ serves as the Lagrange multiplier of the $i$-th constraint and $\lambda=(\lambda_1,\dots,\lambda_N)^T$ represents the vector of these $N$ multipliers.
When strong duality holds in constrained RL problems \cite{oh2017zero}, Problem \ref{Prob_CMDP_theta} is equivalent to the following min-max policy optimization problem:
\begin{equation}
\label{Prob_CMDP_theta_minimax}
 \min_{\lambda\succcurlyeq 0} \max_{\theta\in\Theta} L_0(\lambda,\theta).
\end{equation}
Now these primal-dual methods solve Problem \ref{Prob_CMDP_theta_minimax} with classical projected gradient descent-ascent algorithms.
Specifically, denoting the policy parameters and Lagrange multipliers at the $t$-th step by $\theta^{(t)}$ and $\lambda^{(t)}$,  we compute $(\theta^{(t+1)},\lambda^{(t+1)})$ by
\begin{equation*}
\begin{aligned}
\theta^{(t+1)}&=\texttt{Proj}_\Theta(\theta^{(t)}+\eta_\theta w^{(t)}),~\\
    \lambda^{(t+1)}&=\texttt{Proj}_\Lambda(\lambda^{(t)} - \eta_\lambda l^{(t)}).
\end{aligned}
\end{equation*}
Here $\Theta$ is the space of policy parameters, $\Lambda$ is the set containing valid Lagrange multipliers, $\eta_\theta,~\eta_\lambda$ respectively represent the learning rates of policy parameters and Lagrange multipliers, and $(l^{(t)},w^{(t)})$ denote the gradient directions w.r.t $(\lambda,\theta)$ obtained from certain policy optimization methods.
\citet{tessler2018reward} directly took $(\nabla_\lambda L_0(\lambda^{(t)},\theta^{(t)}),\nabla_\theta L_0(\lambda^{(t)},\theta^{(t)}))$ as $(l^{(t)},w^{(t)})$, while \citet{ding2020natural} set $w^{(t)}$ as the natural policy gradient of $\theta$  and set $l_{(t)}$ as the clipped gradient of $\lambda$.
\citet{tessler2018reward} stated that such gradient descent-ascent algorithms converge to 
\begin{equation*}
(\lambda^*,\theta^*)=\arg \min_{\lambda\succcurlyeq 0} \max_{\theta\in\Theta} L_0(\lambda,\theta),
\end{equation*}
where $\theta^*$ solves Problem \ref{Prob_CMDP_theta} and $\pi_{\theta^*}$ represents the optimal policy when strong duality holds.

\section{Federated Reinforcement Learning with constraint heterogeneity}
We formulate FedRL with constraint heterogeneity as the tuple $\langle \SM,\AM,r,\{(c_i,d_i)\}_{i=1}^N,\gamma,\PM, \{\Gamma_i\}_{i=1}^N\rangle$.
The learning goal of the problem is the same with that of CMDPs shown in Problem \ref{Prob_CMDP}, \ie, finding the policy maximizing the cumulative reward function while satisfying all $N$ constraints.
Despite such similarity, FedRL with constraint heterogeneity has a totally different interpretation of $\{(c_i,d_i)\}_{i=1}^N$ due to the heterogeneity introduced by $\{\Gamma_i\}_{i=1}^N$:
$N$ agents are separately located in $N$ environments, and the $i$-th agent only has access to constraint functions $\{(c_j,d_j)\}_{j\in\Gamma_i}$ in addition to the reward function $r$.
The introduction of $\{\Gamma_i\}_{i=1}^N$ discriminates these $N$ agents from the omniscient agent with full access to $\{(c_i,d_i)\}_{i=1}^N$ in CMDPs.
Without loss of generality, we view constraint functions accessible to the $i$-th constraint function as a whole one and set $\Gamma_i$ to be $\{i\}$ in the following discussion.

\subsection{Local Lagrange functions from $L_0(\lambda,\theta)$}
Given the similar learning goal with CMDPs, it is natural for us to expect that primal-dual methods would work in solving FedRL problems with multiple constraints.
However, constraint heterogeneity makes it impossible for the $i$-th agent to evaluate $L_0(\lambda,\theta)$ from its own experience, because the $i$-th agent cannot collect any information about constraint functions of other agents.
To address the issue, we decompose $L_0(\lambda,\theta)$ into $N$ {local Lagrange functions} $\{L_i(\lambda_i,\theta)\}_{i=1}^N$ as follows:
\begin{align*}
L_0(\lambda,\theta)&=\sum_{i=1}^N L_i(\lambda_i,\theta),\\
    L_i(\lambda,\theta)&=\frac{1}{N}J_r(\theta)+\lambda_i(d_i-J_{c_i}(\theta)),~\forall i\in[N]
\end{align*}
where $L_i(\lambda,\theta)$ is composed of reward function $r$ and the $i$-th constraint function $(c_i,d_i)$, which are both observable for the $i$-th agent.
Moreover, $L_i(\lambda_i,\theta)$ shares a similar formulation with $L_0(\lambda,\theta)$, where the constraint function is multiplied by a non-negative multiplier $\lambda_i$ and then added to a function related to $r$.
We denote Lagrange multipliers $\lambda$ and policy parameters $\theta$ of the $i$-th agent at the $t$-th iteration by $\lambda_i^{(t)}$ and $\theta_i^{(t)}$.
Based on the $i$-th local Lagrange function $L_i(\lambda_i,\theta)$ of $L_0(\lambda,\theta)$, the $i$-th agent is able to apply gradient descent-ascent in updating $(\lambda_i^{(t)},\theta_i^{(t)})$ to $(\lambda_i^{(t+1)},\theta_i^{(t+1)})$.
In this way, agents manage to use primal-dual methods in updating their local policies and Lagrange multipliers based on local Lagrange functions rather than the original Lagrange function.
\begin{algorithm}
    \caption{Federated primal-dual policy optimization}
    \begin{algorithmic}
    \label{Alg_FedPD}
    \small
    \STATE \textbf{Initialize:} Initial parameters $\theta^{(0)}$ and multipliers $\lambda^{(0)}$; Learning rate $\eta_\theta$ and $\eta_\lambda$;
    Projection set $\Lambda$ and $\Theta$.
    \STATE Set $t=0,~\lambda_i^{t}=\lambda^{(0)},~\forall i\in\{1,\dots,N\}$.
    \WHILE{$t<T$}
        \STATE Set $\theta_i^{(t)}=\theta^{(t)}$, $~i\in\{1,2,\dots,N\}$.
        \FOR{ $e=1$ {\bfseries to} $E$}
            \FOR{ $i=1$ {\bfseries to} $N$}
                \STATE Collect experience $\mathcal{T}_{i}^{(t)}$ based on $\pi_{\theta_{i}^{(t)}}$.
                \STATE Policy evaluation based on local experience $\gT_{i}^{(t)}$: $\hat{J}_{c_i}(\theta_i^{(t)})\approx J_{c_i}(\theta_i^{(t)}),~\hat{J}_{r}(\theta_i^{(t)})\approx J_{r}(\theta_i^{(t)})$.
                \STATE Take one-step policy optimization towards maximizing ${L}_i(\lambda_i^{(t)},\theta_i^{(t)})$ as $\texttt{Proj}_\Theta(\theta_{i}^{(t)}+\eta_\theta\hat w_i^{(t)})\rightarrow\theta_{i}^{(t+1)}$, where $\hat w_i^{(t)}$ is obtained based on $\gT_i$.
                \STATE Update multipliers as $\texttt{Proj}_\Lambda(\lambda_{i}^{(t)}-\eta_\lambda\hat{l}_i^{(t)})\rightarrow\lambda_{i}^{(t+1)}$, where $\hat{l}_i^{(t)}=\hat{J}_{c_i}(\theta_i^{(t)})-d_i$
            \ENDFOR       
            \STATE $t=t+1$
        \ENDFOR
        \STATE $N$ agents communicate $\{\theta_i^{(t)}\}_{i=1}^N$ and $\{\lambda_i^{(t)}\}_{i=1}^N$.
        \STATE Set $\theta^{(t)} = \texttt{Aggregate}_\theta(\{\lambda_i^{(t)}\}_{i=1}^N,\{\theta_i^{(t)}\}_{i=1}^N)$.
        \STATE Set $\lambda_{i}^{(t+1)} = \lambda_i^{(t)},~i\in\{1,2,\dots,N\}$.
    \ENDWHILE
    \end{algorithmic}
\end{algorithm}

\subsection{Periodic communication of local policies}
Periodic communication is necessary for policy learning in the federated setting.
However, any exchange of constraint-related experience violates agents' privacy in our case.
Instead, our proposed algorithms organise the periodic communication among these $N$ agents at a policy level.
Assume $E$ to be the number of time steps between adjacent communication rounds.
After the $N$ agents update their local policies for $E$ steps, they communicate their policy parameters and one takes a global aggregation as follows:
\begin{align*}
\theta^{(t)} = \texttt{Aggregate}_\theta&(\{\lambda_i^{(t)}\}_{i=1}^N, \{\theta_i^{(t)}\}_{i=1}^N),
\end{align*}
where $\texttt{Aggregate}_\theta$ stands for any feasible aggregation method, such as averaging in parameter space with uniform weights, \ie, $~\theta^{(t)}=\frac{1}{N}\sum_{i=1}^N\theta_i^{(t)}$.
Due to our assumption on constraint heterogeneity, the algorithm does not update Lagrange multipliers in rounds of communication.

We are ready to give a full implementation of these federated primal-dual policy optimization methods in Alg. \ref{Alg_FedPD}.
Different policy optimization methods lead to different instances of our algorithms.
For detailed theoretical analysis and empirical evaluation, we focus on the following two algorithms: for tabular environments {FedNPG} applies natural policy gradient (NPG) for policy updates of local agents; for non-tabular environments {FedPPO} conducts proximal policy optimization (PPO) on policies parameterized by deep networks.
Details of implementation are left in Appendix \ref{app_ALG}.

\section{Theoretical Analysis}
In this section, we present a theoretical analysis of {FedNPG} on its convergence performance.
In summary, we show FedNPG achieves an $\widetilde O(1/\sqrt{T})$ convergence rate.
Here $\widetilde O$ means we discard any terms of $\operatorname{poly}(\log(\cdot))$ orders.
\subsection{The FedNPG algorithm}
We first give a brief description  of how FedNPG works.
Ideally, we would like to use natural policy gradients to update the local policies, \ie,
$$
\hat w_i^{(t)}=\frac{1}{N}F(\theta_i^{(t)})^\dagger\nabla_\theta J_r(\theta_i^{(t)})-\lambda_i^{(t)}F(\theta_i^{(t)})^\dagger\nabla_\theta J_{c_i}(\theta_i^{(t)}),
$$
where $(\cdot)^{\dagger}$ denotes the matrix pseudoinverse,  
$$
\begin{aligned}
    F(\theta):=&\EB_{s\sim d^{\pi_\theta},a|s\sim\pi_\theta}[\nabla_\theta\pi_\theta(a|s)\nabla_\theta\pi_\theta(a|s)^\top],\\
    d^{\pi_\theta}(s):=&(1-\gamma)\sum_{t=0}^\infty\gamma^t\PB_{s_0\sim\rho,\pi}(s_t=s).
\end{aligned}
$$
However, in practical applications it is computationally expensive to evaluate $F(\theta)$, and $\nabla_\theta J_r(\theta), \nabla_\theta J_{c_i}(\theta)$ are usually unknown.
Therefore, we instead use sample-based NPG \cite{Agarwal2021jmlr} to update the local policies.
That is,
$$
\begin{aligned}
    \hat w_i^{(t)}&=\frac{1}{N}\hat w_r^{(t)}(i)-\lambda_i^{(t)} \hat w_{c_i}^{(t)}.
\end{aligned}
$$
And $ \hat w_r^{(t)}(i)$, $\hat w_{c_i}^{(t)}$ are obtained by solving the following optimization problems with SGD, respectively:
$$
\begin{aligned}
     \hat w_r^{(t)}(i)&\approx\arg\min_w E^{\nu^{(t)}_i}(r,\theta_i^{(t)},w),\\
    \hat w_{c_i}^{(t)}&\approx\arg\min_w E^{\nu^{(t)}_i}(c_i,\theta_i^{(t)},w).
\end{aligned}
$$
Here the $E^{\nu}(\diamond,\theta,w)$ are called the transferred compatible function approximation errors, which are defined as:
\begin{equation*}
E^{\nu}(\diamond,\theta,w):=\EB_{(s,a)\sim\nu}(A^{\pi_\theta}_\diamond(s,a)-w^\top\nabla_\theta\log\pi_\theta(a|s))^2,
\end{equation*}
where $\nu^{(t)}_i(s,a):=\pi_{\theta_i^{(t)}}(a|s)d^{\pi_{\theta_i^{(t)}}}(s)$ is the state-action occupancy measure induced by $\pi_{\theta_i^{(t)}}$.
In terms of the policy aggregation, we consider the averaging in parameter space with uniform weights in the following discussion.

\subsection{Technical assumptions}
Before presenting our main result, we firstly state our technical assumptions.
Note that the following assumptions are standard in the literature of policy optimization \cite{Agarwal2021jmlr, ding2020natural}.

\begin{asmp}[Differentiable policy class]\label{Assumption_differentiable}
We consider a parametrized policy class $\Pi_\theta=\{\pi_\theta|\theta\in\Theta\}$, such that for all $s\in\SM$,~$a\in\AM$, $\log\pi_\theta(s|a)$ is a differentiable function of $\theta$.
\end{asmp}

\begin{asmp}[Lipschitz policy class]\label{Assumption_Lipschitz}
For all $s\in\SM$,~$a\in\AM$, $\log\pi_\theta(s|a)$ is a $L_\pi$-Lipschitz function of $\theta$, i.e.,
$$
\|\nabla_\theta\log\pi_\theta(s|a)\|_{2}\leq L_\pi,\forall s\in\SM,a\in\AM,\theta\in \RB^d.
$$
\end{asmp}

\begin{asmp}[Smooth policy class]\label{Assumption_smooth}
For all $s\in\SM$,~$a\in\AM$, $\log\pi_\theta(s|a)$ is a $\beta$-smooth function of $\theta$, i.e., $\forall s\in\SM,a\in\AM,~\theta,\theta^\prime\in \RB^d$,
$$
\|\nabla_\theta\log\pi_\theta(s|a)-\nabla_\theta\log\pi_{\theta^\prime}(s|a)\|_{2}\leq \beta\|\theta-\theta^\prime\|_2.
$$
\end{asmp}

\begin{asmp}[Positive definite Fisher information]\label{Assumption_pd_Fisher}
$$
F(\theta) \preccurlyeq G^{2} I_{d} \; \text { for any } \; \theta \in \RB^{\mathrm{d}} \text {. }
$$
\end{asmp}
\begin{asmp}[Bounded estimation error]\label{Assumption_est_err}
For any $t\in\{1,...,T\}$, for each $i\in\{1,...,N\}$,
$$
\begin{aligned}
&\left\|\underset{w}{\mathrm{argmin}} E^{\nu_i^{(t)}}(r,\theta_i^{(t)},w)\right\|_2^2\leq W^2,~~\\
&\left\|\underset{w}{\mathrm{argmin}} E^{\nu_i^{(t)}}(c_i,\theta_i^{(t)},w)\right\|_2^2\leq W^2.
\end{aligned}
$$
Also,
$$
\EB \|\hat w_r^{(t)}(i)\|_2^2\leq W^2,~~\EB \|\hat w_{c_i}^{(t)}\|_2^2\leq W^2.
$$
\end{asmp}
In addition, we assume the parametrization $\pi_\theta$ realizes good function approximation in terms of transferred compatible function approximation errors, which can be close to zero as long as the policy class is rich \cite{wang2019neural} or the underlying MDP has low-rank structure \cite{jiang2017contextual}.
\begin{asmp}[Bounded function approximation error]\label{Assumption_func_approx_err}
The transferred compatible function approximation errors satisfies that $\forall t\in\{1,...,T\},~\forall i\in\{1,...,N\}$,
$$
\begin{aligned}
\min_w E^{\nu_i^{(t)}}(r,\theta_i^{(t)},w)\leq \epsilon_{\text{bias}},\\
\min_w E^{\nu_i^{(t)}}(c_i,\theta_i^{(t)},w)\leq \epsilon_{\text{bias}}.
\end{aligned}
$$
\end{asmp}

We also make the following two assumptions to ensure strong duality and boundedness of dual variables.
\begin{asmp}\label{Assumption_full_policy_class}
Assume $\PM(\SM)^{\AM}\subset\bar {\Pi}_\theta$. Here we use $\bar {\Pi}_\theta$ to denote the closure of set ${\Pi}_\theta$. 

\end{asmp}

\begin{asmp}[Slater's condition]\label{Assumption_Slater}
There exist $\xi >0$ and $\tilde\pi\in\Pi$ such that $J_{c_i}(\tilde\pi)+\xi\leq d_i$, $\forall i\in\{1,...,N\}$.
\end{asmp}

\begin{lem}[Strong duality and boundedness of dual variables]
\label{Lemma_strong_duality_bounded_dualvar}
If Assumption~\ref{Assumption_full_policy_class} and Assumption~\ref{Assumption_Slater} are true, we have:
\begin{enumerate}
    \item[(a)] $J_r(\pi^*)=\sup_\theta\inf_\lambda L_0(\theta,\lambda)$;
    \item[(b)] $\sum_{i=1}^N\lambda^*_i\leq  \frac{J_r(\pi^*)-J_r(\tilde\pi)}{\xi}\leq \frac{1}{(1-\gamma)\xi}$. 
\end{enumerate}
\end{lem}

The proof is in Appendix~\ref{Appendix_proof}.
Assumption~\ref{Assumption_full_policy_class} may seem stringent. 
However, it is necessary for the strong duality to hold \citep{paternain2019constrained}, which is of vital importance for the theoretical analysis of primal-dual type methods.
One may notice that \citet{ding2020natural} gave the convergence rate of the NPG-PrimalDual algorithm in the case that the strong duality does not hold (see Theorem 3 in \cite{ding2020natural}).
However, their theoretical analysis relies on the assumption that $\{\lambda^{(t)};t=1,...,T\}$ is a bounded sequence (see Assumption 4 in \cite{ding2020natural}), which is unlikely to be true as long as the strong duality does not hold.
Also, Assumption~\ref{Assumption_full_policy_class} does not implies $\epsilon_{bias}\equiv0$, please see Appendix~\ref{Appendix_example} for an example.

\subsection{Main results and discussions}
Now we present our main result.
The proof is in Appendix~\ref{Appendix_proof}.
\begin{thm}\label{Theorem_main}
Suppose Assumptions~\ref{Assumption_differentiable}-~\ref{Assumption_Slater} are true. $\{\theta^{(t)};t=1,...,T\}$ and $\{\lambda^{(t)};t=1,...,T\}$ are generated by the Fed-NPG algorithm with $\eta_\theta=O(1/N\sqrt{T}),\eta_\lambda=O(1/\sqrt{T})$. Then for the policy $\hat\pi$ returned by the FedNPG algorithm,
$$
\begin{aligned}
&\EB(J_r(\pi^*)-J_r(\hat\pi))=\EB\left[\frac{1}{T}\sum_{t=0}^{T-1}(J_r(\pi^*)-J_r(\pi_{\theta^{(t)}}))\right]\\
&=\widetilde O\left(\frac{|\AM|EN}{\sqrt{T}(1-\gamma)^3}\right)+\widetilde O\left(\frac{\sqrt{\epsilon_{bias}+d/K}}{(1-\gamma)^{3.5}}\right),\\
&\EB(J_{c_i}(\hat\pi)-d_i)=\sum_{i=1}^N\EB\left[\frac{1}{T}\sum_{t=0}^{T-1}J_{c_i}(\pi_{\theta^{(t)}})-d_i\right]_+\\
&=\widetilde O \left(\frac{|\AM|EN}{\sqrt{T}(1-\gamma)^3}\right)+\widetilde O\left(\frac{\sqrt{\epsilon_{bias}+d/K}}{(1-\gamma)^{2.5}}\right).\\
\end{aligned}
$$
$K$ is a hyper-parameter controlling the number of data points we collect at each iteration.
And the definition of $K$ can be found in Algorithm \ref{Alg_FedNPG} in Appendix \ref{app_ALG}.
\end{thm}

\textbf{\textit{Remark}} 1: 
Even if we are allowed to perform an infinitely large number of iterations i.e. $T=\infty$, the error bound in Theorem \ref{Theorem_main} still remains non-zero.
There are two reasons for this phenomenon:
\begin{itemize}
    \item our parameterization introduces inherent biases, which are measured by the transferred compatible function approximation errors;

    \item the natural policy gradient we use is not obtained with an oracle but estimated from a limited number of data points.
\end{itemize}
The meaning of the $\widetilde O(1/\sqrt{T})$ convergence rate of the FedNPG algorithm is that the excess risk would converge to $0$ with an $\widetilde O(1/\sqrt{T})$ rate if our parameterized policy class admits no transferred compatible function approximation errors and an oracle for exact natural policy gradients is available.
Such dependence on $T$ matches error bounds of single-agent algorithm for constrained reinforcement learning \cite{ding2020natural}, as well as results of federated reinforcement learning algorithms in non-constrained cases \cite{jin2022federated}.

\textbf{\textit{Remark}} 2:
Unlike many existing works on federated learning, our finite-sample bounds contain a $O(N)$ factor, meaning that when the number of agent $N$ increases our algorithm would take more time to converge.
Traditional federated settingi considers the average of $N$ \emph{averaged} heterogeneous objectives, while our setting
focuses on $N$ distinct heterogeneous constraints in addition to a homogeneous objective without averaging on these constraint signals.
Moreover, it is worthy to note that even in the setting of constrained RL with an omniscient agent having information of both reward functions and $N$ cost functions, the finite-sample convergence rate will still scale up with $N$ \cite{liu2021policy, li2021faster, zeng2022finite}.

\section{Empirical Study}
In this section, we evaluate the training performance of {FedNPG} and {FedPPO}, in which the policy gradient optimizers used in local updates are respectively natural policy gradient (NPG) and proximal policy optimization (PPO) \footnote{See https://github.com/grandpahao/FedCMDP.git}.

\subsection{The Set-up}
\textbf{Environments.} 
We construct a collection of FedRL tasks with multiple constraints represented by different cost functions in both tabular and non-tabular environments.
We consider two types of tabular environments: {RandomMDP} with randomly generated matrices as different cost functions; {WindyCliff} \citep{paul2019fingerprint} with a sequence of hazard spaces to which the entrance induces a cost.
In terms of non-tabular environments, we modify several classical learning problems in  Gym \citep{1606.01540} as follows: learning agents of {CartPole} and {Inverted-Pendulum} are penalized when their horizontal positions fall into certain regions, and certain actions of agents in {Acrobot} are prohibited at a certain range of states.

\textbf{Policy Parameterization.}
We apply two ways of policy parameterization in instantiating our methods according to the type of environment.
In tabular environments, we use softmax parameterization.
In terms of the policy aggregation $\texttt{Aggregate}_\theta$, we apply an averaging strategy on the level of policy as follows:
\begin{align*}
    \bar{\pi}_{t}(a|s)=\frac{1}{N}\sum_{i=1}^N\pi_{\theta_i^{(t)}}(a|s),\\
    {\theta}^{(t)}(a|s) = \log \bar{\pi}_t(a|s)+C_s,
\end{align*}
where $C_s$ is taken as $\sum_{a\in\SM} \log\bar{\pi}_t(a|s)$ in our methods.
In non-tabular environments, the policy $\pi_\theta$ is parameterized with deep neural networks.
In terms of the policy aggregation, the averaging policy is conducted on the level of deep models, \ie, $~{\theta}^{(t)}=\frac{1}{N}\sum_{i=1}^N \theta_i^{(t)}$.

\textbf{Baselines.} 
For each environment, our methods are compared with two types of baselines: firstly, RL agents trained on locally collected experience without communication, which we refer to  as {NPG}$_k$ ({PPO}$_k$) in the $k$-th environment; secondly, an omniscient agent trained on trajectories with information of both reward functions and $N$ cost functions, which we refer to as {NPG}$_o$ ({PPO}$_o$).
Comparison with the first-type baselines is to show that no local agent is able to independently solve the learning problem with multiple constraints, while the second-type baseline gives us a perspective on how well our methods perform in solving the constrained learning problem.

\textbf{Comparison Metric.}
In most of involved FedRL tasks, we directly display averaged training performances of different algorithms in terms of both reward function and cost functions, and compare them with constraint thresholds.
However, such comparison is infeasible in randomly generated {RandomMDP} environments, because reward functions and cost functions vary among different instances.
Instead, we introduce three auxiliary ratios which unify such misalignment in different {RandomMDP} environments, \ie, Reward Ratio (\textit{RR}), maximum Violation Ratio (\textit{mVR}) and maximum Relative Violation Ratio (\textit{mRVR}), which are defined as follows: 
\begin{align*}
    \textit{RR}(\pi)&=J_r(\pi)/J_r(\pi_o),\\
    \textit{mVR}(\pi) &= \max_{i\in[N]} J_{c_i}(\pi)/d_i,\\
    \textit{mRVR}(\pi)&=\max_{i\in[N]} J_{c_i}(\pi)/J_{c_i}(\pi_o),
\end{align*}
where $\pi$ represents any convergent policy to be evaluated and $\pi_o$ represents the convergent policy of the corresponding second-type baseline with full access to the $N$ constraints.

\textbf{Other Details.} In terms of other experiment settings on environment construction and hyperparameter selection, we leave a detailed description in Appendix \ref{app_EXP}.
\begin{figure*}[!htb]
    \centering
    \includegraphics[width=0.9\textwidth]{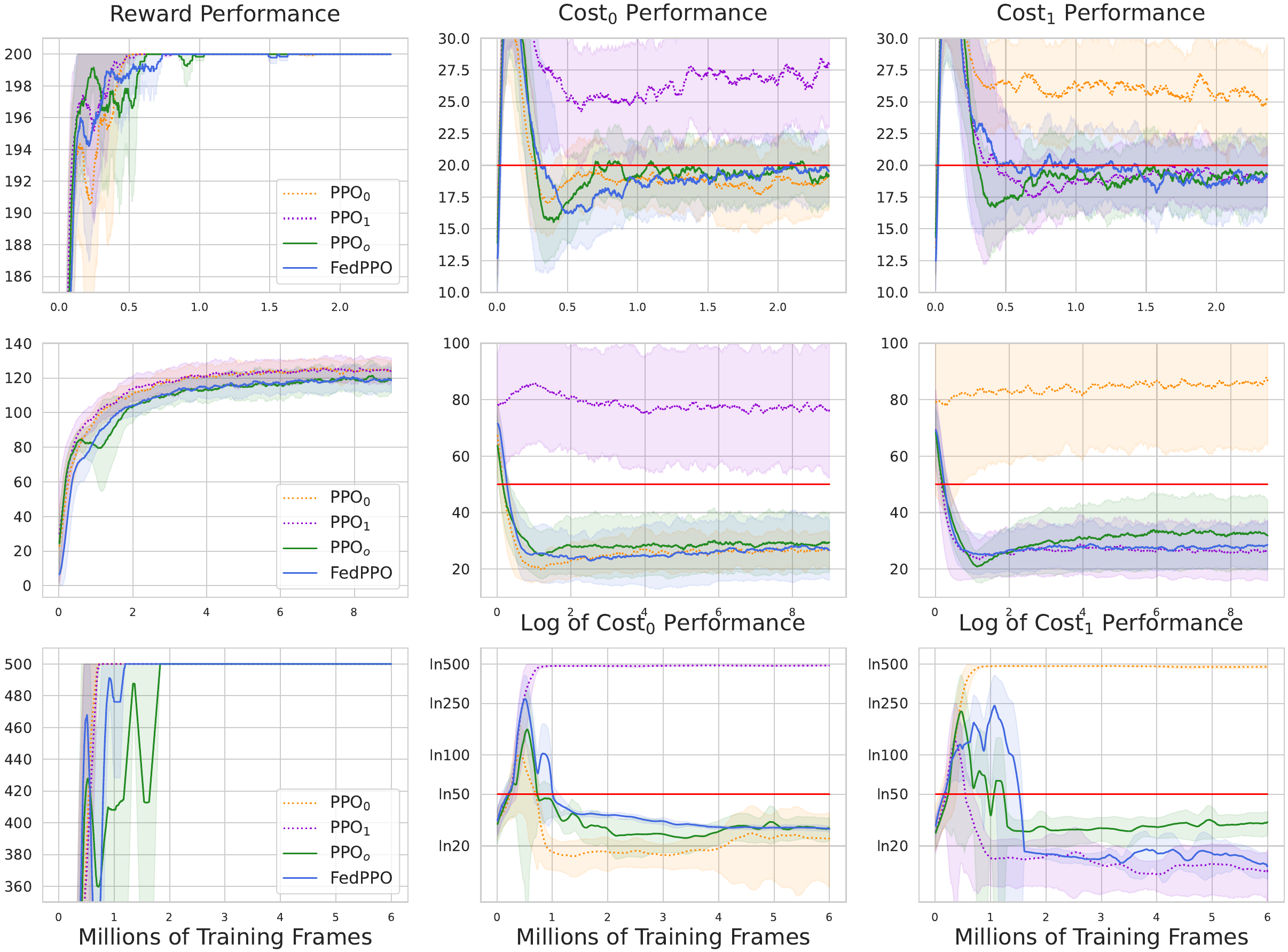}

    \caption{Comparison between baselines and FedPPO in {CartPole} (first row), {Acrobot} (second row) and {Inverted-Pendulum} (third row): we depict the mean as line, standard error as shadow, and constraint thresholds as red lines.}
    \label{fig:deep_exp}
\end{figure*}

\subsection{Performance of FedNPG}
To justify the convergence of FedNPG, we evaluate its performance in FedRL tasks with tabular environments, \ie, {RandomMDP} and {WindyCliff}.
The first row of Figure \ref{fig:tabular_exp} reveals that FedNPG and NPG$_o$ achieve similar performance in terms of both reward performance and cost violation, while NPG$_k$ fails at satisfying all constraints.
In WindyCliff, the second row of Figure \ref{fig:tabular_exp} tells us that FedNPG is comparable with NPG$_o$ in terms of reward performance and satisfies all the constraints with a smaller variance.

\subsection{Performance of FedPPO}
FedPPO utilizes deep neural networks to approximate policy functions and value functions, and we evaluate its performance on FedRL tasks in CartPole, Acrobot and Inverted-Pendulum.
Different environments justify the efficiency of FedPPO from different perspectives: CartPole focuses on the discrete control with state-based constraints; Acrobot additionally considers state-action-based constraints; Inverted-Pendulum focuses on agents with continuous action space. 
In CartPole and Inverted-Pendulum, Figure \ref{fig:deep_exp} reveals that all the methods achieve maximum reward and only convergent performances of PPO$_{o}$ and FedPPO satisfy constraints specified by red lines. 
It is worth noting that PPO$_k$ indeed satisfies the $k$-th constraint but fails in the other constraint.
Such phenomenon is more obviously observed in Inverted-Pendulum.
In Acrobot, FedPPO achieves comparable reward performance with PPO$_o$ and satisfies all the constraints, while PPO$_k$ obviously violates the unobservable constraint.

\section{Conclusion}
We have imposed constraint heterogeneity into FedRL problems, where agents have access to different constraints and the learned policy is expected to satisfy all constraints.
Through constructing local Lagrange functions, agents are able to conduct local updates without knowledge of others' experience.
Together with periodical communication of policies, we have proposed federated primal-dual policy optimization methods to solve FedRL problems with constraint heterogeneity.
Moreover, we have furthermore analyzed two instances of our methods: FedNPG and FedPPO.
FedNPG is proved to achieve an $\tilde{O}(1/\sqrt{T})$ convergence rate, enjoy reasonable sample complexity compared with existing works on constrained reinforcement learning, and achieve performance comparable with an omniscient agent having access to all constraint signals in tabular FedRL tasks.
FedPPO is evaluated on complicated tasks with neural networks as function approximators, and manages to find optimal policies simultaneously satisfying constraints distributed in different agents.

\bibliographystyle{unsrtnat}
\bibliography{references}  

\appendix
\newpage
\section{Implementations for \textbf{FedNPG} and \textbf{FedPPO} }
\label{app_ALG}
Here we display detailed implementations for federated primal-dual natural policy gradient {FedNPG} and federated primal-dual proximal policy optimization {FedPPO}.
\subsection{Federated primal-dual natural policy gradient (FedNPG)}
In {FedNPG}, we apply sampling procedure in Algorithm 3 of \cite{agarwal2019theory} to estimate value functions and leave choices of $\texttt{Aggregate}_{\theta}(\cdot)$ unspecified for different ways of parameterization. 
In addition, we use $\pi_i^{(t)}$ in short of $\pi_{\theta_i^{(t)}}$.
\begin{algorithm}[!htb]
\small
    \caption{Federated primal-dual Natural Policy Gradient (FedNPG)}
    \label{Alg_FedNPG}
    \begin{algorithmic}
        \STATE \textbf{Initialization:} Initial parameters $\theta^{(0)}$ and multipliers $\lambda^{(0)}$; Learning rates $\eta_\theta$, $\eta_\lambda$, $\alpha$; Number of sampled trajectories $K$; Projection sets $\Lambda$ and $\Theta$.
        \STATE Set $t=0$.
        \STATE Set $\lambda_i^{(t)}=\lambda^{(t)},~i\in\{1,\dots,N\}$.
        \WHILE{$t<T$}
            \STATE Set $\theta_i^{(t)}=\theta^{(t)},~i\in\{1,\dots,N\}$.
            \FOR{ $e=1$ {\bfseries to} $E$}
                \FOR{ $i=1$ {\bfseries to} $N$}
                \STATE Set $\widehat{V}^{(t)}_{c_i}(\rho)=0$.
                    \FOR{ $k=0$ {\bfseries to} $K-1$}
                        \STATE Draw $(s,a)\sim \nu^{(t)}_i$, with $\nu^{(t)}_i(s,a)=d^{\pi^{(t)}_i}(s)\pi^{(t)}_i(a|s)$.
                        \STATE Draw $L\sim \text{Geometry}(1-\gamma)$ and execute policy $\pi^{(t)}_i$ from $(s,a)$ for $L$ steps, then construct estimators as $$
                       \widehat Q^{(t)}_r(s,a)=\sum^{L}_{l=0} r(s_l,a_l),\ \widehat Q^{(t)}_{c_i}(s,a)=\sum^{L}_{l=0} c_i(s_l,a_l), \text{where } (s_0,a_0)=(s,a).$$
                        \STATE Draw $L\sim \text{Geometry}(1-\gamma)$ and execute policy $\pi^{(t)}_i$ from $s$ for $L$ steps, then construct estimators as $$\widehat V^{(t)}_r(s)=\sum^{L}_{l=0} r(s_k,a_k),\ \widehat V^{(t)}_{c_i}(s)=\sum^{L}_{l=0} c_i(s_l,a_l), \text{where } s_0=s.$$
                        
                        \STATE Set $\widehat A_r^{(t)}(s,a)=\widehat Q^{(t)}_r(s,a)-\widehat V_r^{(t)}(s)$ and $\widehat A_{c_i}^{(t)}(s,a)=\widehat Q^{(t)}_{c_i}(s,a)-\widehat V_{c_i}^{(t)}(s)$.
                        \STATE Update $w_r^{(k+1)}(i)=w_r^{(k)}(i)-\alpha G_r^{(k)}(i)$, $w_{c_i}^{(k+1)}=w_{c_i}^{(k)}-\alpha G_{c_i}^{(k)}$ with $$
                       \begin{aligned}
                          G_r^{(k)}(i)&=2({w_r^{(k)}(i)}^\top\nabla_\theta\log\pi^{(t)}_i(a|s)-\widehat A_r^{(t)}(s,a))\nabla_\theta\log\pi^{(t)}_i(a|s),\\ G_{c_i}^{(k)}&=2({w_{c_i}^{(k)}}^\top\nabla_\theta\log\pi_i^{(t)}(a|s)-\widehat A_{c_i}^{(t)}(s,a))\nabla_\theta\log\pi^{(t)}_i(a|s)
                       \end{aligned}$$
                       \STATE Draw $s\sim \rho$, $L\sim\text{Geometry}(1-\gamma)$ and execute policy $\pi_i^{(t)}$ from $s$ for $L$ steps, then update $\widehat{V}_{c_i}(\rho)$ as
                       $$
                       \widehat{V}^{(t)}_{c_i}(\rho)=\widehat{V}^{(t)}_{c_i}(\rho)+\frac{1}{K}\sum_{l=0}^{L-1}c_i(s_l,a_l),\ \text{where\ }s_0=s. 
                       $$
                    \ENDFOR
                    \STATE Set $\hat w_r^{(t)}(i)=\frac{1}{K}\sum_{k=1}^K w_r^{(k)}(i)$, $\hat w_{c_i}^{(t)}=\frac{1}{K}\sum_{k=1}^K w_{c_i}^{(k)}$, $\hat w^{(t)}(i)= \hat w_r^{(t)}(i)/N-\lambda_i^{(t)}\hat w_{c_i}^{(t)}$.
                    \STATE Update parameters and multipliers as $$\theta^{(t+1)}_i=\operatorname{Proj}_{\Theta}(\theta^{(t)}_i+\eta_\theta \hat w^{(t)}(i)),\ \lambda_i^{(t+1)}=\operatorname{Proj}_{\Lambda}(\lambda_i^{(t)}-\eta_{\lambda}(d_i-\widehat V_{c_i}^{(t)}(\rho))).$$
                \ENDFOR
                \STATE Set $t=t+1$.
            \ENDFOR
            \STATE $N$ agents communicate $\{\theta_i^{(t)}\}_{i=1}^N$ and $\{\lambda_i^{(t)}\}_{i=1}^N$.
            \STATE Set $\theta^{(t)} = \texttt{Aggregate}_\theta(\{\lambda_i^{(t)}\}_{i=1}^N,\{\theta_i^{(t)}\}_{i=1}^N)$.
        \ENDWHILE  
        \STATE \textbf{RETURN} $\hat\pi=\pi_{\hat\theta}$, where $\hat\theta\sim\operatorname{Unif}(\{\theta^{(1)},...,\theta^{(T)}\})$.


    \end{algorithmic}
\end{algorithm}
\subsection{Federated primal-dual proximal policy optimization (FedPPO)}
{FedPPO} is implemented for complicated tasks with large state space or continuous action space.
In these cases, an policy $\pi$ is usually parameterized with a deep neural network $\theta$.
In addition to the policy network, {FedPPO} also utilizes deep neural networks $\phi,\psi$ to estimate value functions for both reward signals, \ie $~V_\phi$, and cost signals, \ie $~V_\psi$.
It is noteworthy that $V_\psi$ contains private information of constraint signals and is prohibited from any communication.
Value networks $V_\phi$ of reward functions is communicated along with policy networks $\pi_\theta$ and $\texttt{Aggregate}_\phi$ has the same formulation with $\texttt{Aggregate}_\theta$.
\begin{algorithm}[!htb]
\label{Alg_FedPPO}
    \caption{Federated primal-dual Proximal Policy Optimization}
    \begin{algorithmic}
    \STATE \textbf{Initialization:} Initial policy parameters $\theta^{(0)}$, critic parameters $(\phi^{(0)}$,$~\psi^{(0)})$, multipliers $\lambda^{(0)}$; Learning rates $\eta_\theta$, $\eta_\lambda$, $\eta_\phi$; Length of sampled trajectory $K$; Number of inner iterations $K_{in}$; Projection sets $\Lambda$ and $\Theta$.
    \STATE Set $t=0$.
    \STATE Set $\lambda_i^{(t)}=\lambda^{(t)},~\psi_i^{(0)}=\psi^{(0)},~i\in\{1,\dots,N\}$.
    \WHILE{ $t<T$ }
        \STATE Set $\theta_i^{(t)}=\theta^{(t)},\phi_i^{(t)}=\phi^{(t)},~i\in\{1,\dots,N\}$.
        \FOR{ $e=1$ {\bfseries to} $E$ }
            \FOR{ $i=1$ {\bfseries to} $N$ }
                \STATE Collect a trajectory of length $\gT_{k}^t=\{(s^l,a^l,r^l,c^l_i,d^l)\}_{l=0}^{K-1}$ following $\pi_{\theta_i^{(t)}}$.
                \STATE Compute reward-to-go $\{R^l\}_{l=0}^{K-1}$ and cost-to-go $\{C^l_i\}_{l=0}^{K-1}$.
                \STATE Compute averaged cost of one episode: $\hat{J}_{C_i}=(\sum_{l=0}^{K-1}C_i^l\mathds{1}_{d^l=1})/(\sum_{l=0}^{K-1}\mathds{1}_{d^l=1})$.
                \STATE Take one-step gradient descent w.r.t. $\lambda$: $\lambda_i^{(t+1)}=\operatorname{Proj}_\Lambda(\lambda_i^{(t)}-\eta_\lambda(d_i-\hat{J}_{C_i}))$.
                \STATE Compute squared errors of $V_{\phi_{j-1}}$ and $V_{\psi_{j-1}}$: 
                    $$
                    Error_R(\phi) = \frac{1}{K}\sum_{l=0}^{K-1}(R^l-V_\phi(s^l))^2,~Error_{C_i}(\psi) = \frac{1}{K}\sum_{l=0}^{K-1}(C_i^l-V_\psi(s^l))^2.$$
                \STATE
                \STATE Compute advantages of both reward and cost functions with local critic networks $V_{\phi_i^{t}}$ and $V_{\psi_i^{t}}$:
                $$
                    A_r^l=R^l-V_{\phi_i^{(t)}}(s^l),~A_{c_i}^l=C^l_i-V_{\psi_i^{(t)}}(s^l),~\forall l\in\{0,\dots,K-1\}.
                $$
                \STATE Take one-step gradient descent w.r.t. $\phi$: $\phi_i^{(t+1)}=\phi_i^{(t)}-\eta_\phi\frac{\partial_\phi Error_R(\phi)}{\partial_\phi}\big |_{\phi=\phi_i^{(t)}}$.
                \STATE Take one-step gradient descent w.r.t. $\psi$: $\psi_i^{(t+1)}=\psi_i^{(t)}-\eta_\psi\frac{\partial_\psi Error_{C_i}(\psi)}{\partial_\psi}\big |_{\psi=\psi_i^{(t)}}$.
                \STATE Compute advantages of local Lagrange function $A_L^l=A_r^l/N-\lambda_{i}^{(t)}A_{c_i}^l,~\forall l\in\{0,\dots,K-1\}$.
                \STATE Set $\theta_0=\theta_{i}^{(t)}$.
                \FOR{ $j=1$ {\bfseries to} $K_{in}$ }
                    \STATE Construct the PPO-clip objective: $$Clip(\theta)=\sum_{l=0}^{K-1}\min\left ( \frac{\pi_{\theta}(a_l|s_l)}{\pi_{\theta_i^{(t)}}(a_l|s_l)}A_L^l,\max\left ((1-\epsilon)A_L^l,~(1+\epsilon)A_L^l\right )~\right ).
                    $$
                    \STATE Take one-step gradient ascent w.r.t. policy parameters: $\theta_j = \operatorname{Proj}_{\Theta}\left (\theta_{j-1}+\eta_\theta\frac{\partial_\theta Clip(\theta)}{\partial_\theta}\big |_{\theta=\theta_{j-1}} \right )$.
                    
                \ENDFOR
                \STATE Set $\theta_{i}^{(t+1)}=\theta_{K_{in}}$.
                
            \ENDFOR
            \STATE $t=t+1$
        \ENDFOR
        \STATE $N$ agents communicate $\{(\theta_i^{(t)},\phi_i^{(t)})\}_{i=1}^N$ and $\{\lambda_i^{(t)}\}_{i=1}^N$.
        \STATE Set $\theta^{(t)} = \texttt{Aggregate}_\theta(\{\lambda_i^{(t)}\}_{i=1}^N,\{\theta_i^{(t)}\}_{i=1}^N)$ and $\phi^{(t)} = \texttt{Aggregate}_\phi(\{\lambda_i^{(t)}\}_{i=1}^N,\{\phi_i^{(t)}\}_{i=1}^N)$.
    \ENDWHILE
    \end{algorithmic}
\end{algorithm}

\section{Detailed Experiment Settings}
\label{app_EXP}
\subsection{Environment Construction}
\paragraph{RandomMDP}
An instance of {RandomMDP} is composed of a randomly generated transition dynamic, a randomly generated reward function and $N$ randomly generated cost functions.
The constraint threshold is determined with a randomly generated anchor policy: its performances w.r.t. $N$ cost functions multiplied by a hardness factor $\nu<1$ are set to be threshold values of the $N$ constraints.
In our case, we set $|\SM|=3$, $|\AM|=5$, $N=4$ and $\nu=0.7$.

\paragraph{WindyCliff}  
{WindyCliff} is a modified version of a classical example in \cite{sutton1998introduction}: Cliff Walking.
In {WindyCliff}, the agent is expected to walk from the left-bottom grid to the right-bottom with hazard regions known as the cliff zone, and there is a blowing wind in the environment with probability $\theta\leq 1$ of forcing the agent to move regardless of its action.
In our case, the gridworld size is set to be $4\times 10$, the agent can move to any adjacent grid in four directions (stay unmoved at the move beyond boundaries), there are three hazard zones notated as $Z_1,Z_2,Z_3$ and there is a wind blowing to the bottom of our gridworld. 
\begin{figure}[!htb]
    \centering
    \includegraphics[width=0.5\textwidth]{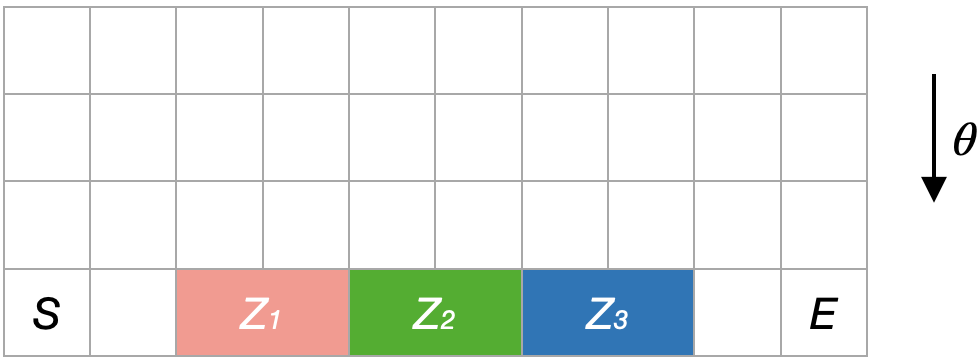}
    \label{fig:app_WindyCliff}
    \caption{The gridworld in \textbf{WindyCliff} with size of $4\times 10$.}
\end{figure}
\newline
In terms of the reward function, we reward the agent with 20 at reaching $E$, while it receives a negative reward of $-1$ at reaching the other grids.
In terms of constraints, we introduce three cost functions corresponding to three hazard zones: the $k$-th cost function generates a cost of $10$ at reaching the grid in $Z_k$, and the threshold of these cost functions is uniformly set to be $1.5$.
In other words, the agent satisfying these three cost functions are prohibited from reaching any grid in hazard zones.
Moreover, the wind intensity is set to be $\theta=0.4$ in our instance.

\paragraph{CartPole} In the original setting of CartPole \cite{1606.01540}, the block is horizontally restricted in a range of $[-4.8,4.8]$ and the agent is rewarded when the pole is kept upright.
Keeping the reward function unchanged, we introduce two hazard zones for the horizontal position of the block: 
\begin{align*}
    Z_1&=[-2.4, -2.3]\cup[-1.3, -1.2]\cup[-0.1, 0.0]\cup[1.1, 1.2]\cup[2.2, 2.3],\\
    Z_2&=[-2.3, -2.2]\cup[-1.2, -1.1]\cup[0.0, 0.1]\cup[1.2, 1.3]\cup[2.3, 2.4].
\end{align*}
In terms of constraints, the $k$-th cost function generates a cost of $1$ when the horizontal position of the block falls in $Z_k$ and constraint budgets are both set to be $20$.

\paragraph{Acrobot}
In our setting, the agent is rewarded with $1.0$ when the free end achieves the target height $H$, and it is otherwise rewarded with $0.001(s_h-H)$ when the height of free end is $s_h$.
In terms of constraints, we introduce two cost functions as follows:
$$
C_1(s,a)=\mathds{1}_{\{s_{\theta_1}\in[-\pi/2,0.0],a=\text{push}\}},~C_2(s,a)=\mathds{1}_{\{s_{\theta_1}\in[-\pi/2,0.0],a=\text{pull}\}},
$$
where $s_{\theta_1}$ represents the angle of the first joint.
Budgets of these constraints are set to be $40$.
In other words, agents satisfying these two actions are not expected to take any action, \ie $~a=\text{null}$, when the first joint falls in the leftbottom region.

\paragraph{InvertedPendulum} In the original setting of {InvertedPendulum} \cite{1606.01540}, the agent is rewarded when the pole is kept upright and there is no limit on the horizontal range of the block.
Keeping the reward function unchanged, we firstly restrict the horizontal range of the block within $[-2.4,2.4]$ and then introduce two hazard zones: 
\begin{align*}
    Z_1&=[-2.4, -2.0]\cup[-1.3, -0.8]\cup[-0.2, 0.2]\cup[0.8, 1.3]\cup[2.0, 2.4],\\
    Z_2&=[-1.9, -1.3]\cup[-0.5, -0.2]\cup[0.2, 0.5]\cup[1.3, 1.9].
\end{align*}
In terms of constraints, the $k$-th cost function generates a cost of $1$ when the horizontal position of the block falls in $Z_k$ and constraint budgets are both set to be $20$.

\subsection{Hyperparameter Selection}
\paragraph{Network structure} In tasks of {CartPole}, {Acrobot} and {InvertedPendulum}, we approximate policy functions and value functions with deep neural networks.
In {CartPole} and {Acrobot}, both policy functions and value functions utilize neural networks with two hidden layers of size $(64,64)$.
In {InvertedPendulum}, both policy functions and value functions utilize neural networks with two hidden layers of size $(256,256)$.

\paragraph{Learning rates} 
In {RandomMDP}, $(\eta_\theta,\eta_\lambda)$ are set to be $(1e-3,1e-3)$.
In {WindyCliff}, $(\eta_\theta,\eta_\lambda)$ are set to be $(3e-4,3e-4)$.
In {CartPole}, {Acrobot} and {InvertedPendulum}, $(\eta_\theta,\eta_\psi,\eta_\phi,\eta_\lambda)$ are set to be $(1e-4,1e-4,1e-4,1e-3)$.

\paragraph{Projection set $\Lambda$} In {RandomMDP} and {WindyCliff}, we set $\Lambda=[0,10]$.
In {CartPole} and {InvertedPendulum}, we set $\Lambda=[0,1]$.
In {Acrobot}, we set $\Lambda=[0,2]$.

\paragraph{Other hyperparameters}
In {RandomMDP} and {WindyCliff}, we set $E=5$ and $K=10$.
In the other tasks, we set $E=1$ and $K=10000$, which indicates that any agent performs at least $10$ local updates between every two communication rounds.

\section{Notations in the Proofs}
In this section, we introduce some useful notations that do not appear in the main paper.
For a CMDP $\langle \SM,\AM,r,\{(c_i,d_i)\}_{i=1}^N,\gamma,\PM\rangle$, we define the value function, state-action value function and advantage function w.r.t the reward $r$ as follows:
$$
\begin{aligned}
V_r^\pi(s) &=\EB\left[\sum_{t=0}^\infty\gamma^t r(s_t,a_t)\mathrel{\bigg|} s_0=s,a_t\sim\pi(\cdot|s_t)\right], \\
Q_r^\pi(s,a)&=\EB\left[\sum_{t=0}^\infty\gamma^t r(s_t,a_t)\mathrel{\bigg|} s_0=s,a_0=a,a_t\sim\pi(\cdot|s_t)\ \text{for}\ {t\geq 1}\right],\\
A^{\pi}_r(s,a)&=Q^{\pi}_r(s,a)-V^\pi_r(s).
\end{aligned}
$$
Given an initial distribution $\rho$, we also define $V_r^\pi(\rho)=\EB_{s\sim\rho}V_r^\pi(s)$.
We define the value function, state-action value function and advantage function w.r.t the $i$th cost function $c_i$ as $V_{c_i}^\pi(s)$, $Q_{c_i}^\pi(s,a)$ and $A^{\pi}_{c_i}(s,a)$ in a similar manner.
Given an initial distribution $\rho$, we also define $V_{c_i}^\pi(\rho)=\EB_{s\sim\rho}V_{c_i}^\pi(s)$.
For simplicity of notations, we use $\pi^{(t)}$ in short for $\pi_{\theta^{(t)}}$, $\pi_i^{(t)}$ in short for $\pi_{\theta_i^{(t)}}$ and $\pi^{[t]}$ in short of $\pi_{\bar\theta^{(t)}}$.
We also use $V_\diamond^{(t)}$, $V_\diamond^{(t)_i}$, $V_\diamond^{[t]}$ and $V_\diamond^{*}$ in short for $V_\diamond^{\pi^{(t)}}$, $V_\diamond^{\pi_i^{(t)}}$, $V_\diamond^{\pi^{[t]}}$ and $V_\diamond^{\pi^*}$ respectively.
Here $\diamond$ represents either the reward $r$ or the $i$th cost function $c_i$.
$Q_\diamond^{(t)}$, $Q_\diamond^{(t)_i}$, $Q_\diamond^{[t]}$, $Q_\diamond^{*}$, $A_\diamond^{(t)}$, $A_\diamond^{(t)_i}$, $A_\diamond^{[t]}$ and $A_\diamond^{*}$ are defined similarly.

\section{Auxiliary Lemmas}
\begin{lem}[Performance difference lemma]\label{Lemma_performance_difference}
For any state $s\in \gS$, any stationary policy $\pi,\pi^\prime$, we have
$$
V^{\pi}_\diamond\left(s_{0}\right)-V^{\pi^{\prime}}_\diamond\left(s_{0}\right)=\frac{1}{1-\gamma} \mathbb{E}_{s \sim d_{s_{0}}^{\pi}}\left[\mathbb{E}_{a \sim \pi(\cdot \mid s)} A_\diamond^{\pi^{\prime}}(s, a)\right].
$$
\end{lem}
\begin{proof}
See \cite{Kakade02approximatelyoptimal}.
\end{proof}

\begin{lem}[Lipschitz values]\label{Lemma_Lipschitz_value}
Let Assumption~\ref{Assumption_Lipschitz} hold true. For any $s_0\in\gS$,
$$
|V_\diamond^{\pi_\theta}(s)-V_\diamond^{\pi_{\theta^\prime}}(s)|\leq \frac{|\gA|L_\pi\|\theta-\theta^\prime\|_2}{(1-\gamma)^2}.
$$
\end{lem}
\begin{proof}
By Lemma~\ref{Lemma_performance_difference} we have
$$
\begin{aligned}
V_\diamond^{\pi_\theta}(s_0)-V_\diamond^{\pi_{\theta^\prime}}(s_0)&=\frac{1}{1-\gamma}\EB_{s\sim d^{\pi_\theta}}\left[\EB_{a\sim\pi_\theta(\cdot|s)}A_\diamond^{\pi_{\theta^\prime}}(s,a)\right]\\
&=\frac{1}{1-\gamma}\EB_{s\sim d^{\pi_\theta}}\left[(\pi_\theta(\cdot|s)-\pi_{\theta^\prime}(\cdot|s))^\top Q_\diamond^{\pi_{\theta^\prime}}(s,a)\right]\\
&\leq \frac{1}{1-\gamma}\EB_{s\sim d^{\pi_\theta}}\left[\|\pi_\theta(\cdot|s)-\pi_{\theta^\prime}(\cdot|s)\|_\infty \|Q_\diamond^{\pi_{\theta^\prime}}(s,a)\|_1\right]\\
&\leq \frac{|\gA|}{(1-\gamma)^2}\max_{s\in\gS,a\in\gA}|\pi_\theta(a|s)-\pi_{\theta^\prime}(a|s)|.
\end{aligned}
$$
Now we may apply Assumption~\ref{Assumption_Lipschitz} to get for any $s\in\gS,a\in\gA$,
$$
\|\nabla_\theta\pi_\theta(s|a)\|_2\leq \|\nabla_\theta\log\pi_\theta(s|a)\|_2\leq L_\pi.
$$
Therefore, for any $s\in\gS,a\in\gA$
$$
|\pi_\theta(a|s)-\pi_{\theta^\prime}(a|s)|\leq L_\pi\|\theta-\theta^\prime\|_2
$$
and
$$
\begin{aligned}
V_\diamond^{\pi_\theta}(s_0)-V_\diamond^{\pi_{\theta^\prime}}(s_0)&\leq\frac{|\gA|}{(1-\gamma)^2}\max_{s\in\gS,a\in\gA}|\pi_\theta(a|s)-\pi_{\theta^\prime}(a|s)|\\
&\leq \frac{|\gA|L_\pi\|\theta-\theta^\prime\|_2}{(1-\gamma)^2}
\end{aligned}
$$
\end{proof}

\begin{lem}[Strong duality and bounded dual variables]
\label{Lemma_strong_duality_bounded_dualvar_AUX}
If Assumption~\ref{Assumption_full_policy_class} and Assumption~\ref{Assumption_Slater} are true, we have:
$$
\begin{aligned}
&& &\text{(a):\ } J_r(\pi^*)=\sup_{\theta^*}\inf_{\lambda}L_0(\theta,\lambda);\\
&& &\text{(b):\ } \sum_{i=1}^N\lambda^*_i\leq  \frac{J_r(\pi^*)-J_r(\tilde\pi)}{\xi}\leq \frac{1}{(1-\gamma)\xi}.
\end{aligned}
$$
\end{lem}
\begin{proof}
The proof of part (a) can be found in~\cite{paternain2019constrained}.
To prove part (b), we consider a sublevel set $\Lambda_a$:
$$
\Lambda_a := \{\lambda\in[0,\infty)^N\mid \sup_\theta L_0(\theta,\lambda)\leq a\}.
$$
For any $\lambda\in\Lambda_a$, we have
$$
J_r(\tilde\pi)+\xi\sum_{i=1}^N\lambda_i\leq J_r(\tilde\pi)+\sum_{i=1}^N\lambda_i(d_i-J_{c_i}(\tilde\pi))\leq a,
$$
where $\tilde\pi$ is defined as in Assumption~\ref{Assumption_Slater}.
Thus we have $\sum_{i=1}^N\lambda_i \leq \frac{a-J_r(\tilde\pi)}{\xi}$.
The result follows by setting $a=J_r(\pi^*)$.
\end{proof}

\begin{lem}[Constraint violation]\label{Lemma_optimization_constraint}
Let Assumption~\ref{Assumption_full_policy_class}, Assumption~\ref{Assumption_Slater} hold true. If there exists a policy $\bar\pi\in\Pi$, a positive scalar $C>2\lambda^*_i,\forall i\in\{1,...,N\}$, and another positive scalar $\delta$ such that:
$$
J_r(\pi^*)-J_r(\bar\pi)+C\sum_{i=1}^N(J_{c_i}(\bar \pi)-d_i)_+\leq \delta,
$$
then we have
$$
\sum_{i=1}^N (J_{c_i}(\bar \pi)-d_i)_+\leq \frac{2\delta}{C}.
$$
\end{lem}
\begin{proof}
For any $\tau\in \mathbb{R}^N$, define the perturbation function associated to Problem~\ref{Prob_CMDP} as:
$$
\begin{aligned}
P(\tau) = \max_{\pi}\ &J_r(\pi)\\
s.t.\ &J_{c_i}(\pi)\leq d_i-\tau_i,\ i=1,...,N.
\end{aligned}
$$
First, according to Lemma~\ref{Lemma_strong_duality_bounded_dualvar_AUX} we have $\forall\pi$,
$$
J_r(\pi)+\sum_{i=1}^N \lambda^*_i(d_i-J_{c_i}(\pi))\leq J_r(\pi^*)=P(0).
$$
For any $\pi\in\Pi$ such that $J_{c_i}(\pi)\leq d_i-\tau_i$, we have
$$
\begin{aligned}
P(0)-\sum_{i=1}^N\tau_i\lambda_i^*&\geq J_r(\pi)+\sum_{i=1}^N \lambda^*_i(d_i-J_{c_i}(\pi))-\sum_{i=1}^N\tau_i\lambda_i^*\\
&=J_r(\pi)+\sum_{i=1}^N \lambda^*_i(d_i-J_{c_i}(\pi)-\tau_i)\\
&\geq J_r(\pi).
\end{aligned}
$$
If we take $\tau_i=-(J_{c_i}(\bar \pi)-d_i)_+$, then
because $J_{c_i}(\bar\pi)\leq d_i+(J_{c_i}(\bar \pi)-d_i)_+$
we can get
$$
\begin{aligned}
    J_r(\bar\pi)\leq P(0)-\sum_{i=1}^N\tau_i\lambda_i^*=J_r(\pi^*)-\sum_{i=1}^N\tau_i\lambda_i^*.
\end{aligned}
$$
Noting that
$$
\frac{C}{2}\sum_{i=1}^N (-\tau_i)\leq \sum_{i=1}^N(C-\lambda^*_i)(-\tau_i)\leq C\sum_{i=1}^N(-\tau_i)+J_r(\pi^*)-J_r(\bar\pi)\leq \delta,
$$
we complete the proof.
\end{proof}

\section{Omitted Proofs}\label{Appendix_proof}
\begin{proof}[Proof of Lemma~\ref{Lemma_strong_duality_bounded_dualvar} ]
See the proof of Lemma~\ref{Lemma_strong_duality_bounded_dualvar_AUX}.
\end{proof}

\begin{lem}\label{Lemma_basic_expansion}
We have
$$
\begin{aligned}
&\EB\left[\frac{1}{TN}\sum_{t=0}^{T-1}\sum_{i=1}^N(V_r^*(\rho)-V_r^{(t)_i}(\rho))\right]\\
&\leq \frac{\log|\gA|}{T(1-\gamma)\eta_\theta}+\EB\left[\frac{1}{TN}\sum_{t=0}^{T-1}\sum_{i=1}^N \frac{\beta}{1-\gamma}\left\|\bar\theta^{(t)}-\theta_i^{(t)}\right\|_2\left\|\hat w_i^{(t)}\right\|_2\right]+\EB\left[\frac{1}{T}\sum_{t=0}^{T-1}\sum_{i=1}^N\lambda_i^{(t)}(V_{c_i}^{*}(\rho)-V_{c_i}^{(t)_i}(\rho))\right]\\
&\quad+\EB\left[\frac{1}{TN(1-\gamma)}\sum_{t=0}^{T-1}\sum_{i=1}^N\sqrt{E^{\nu^*}\left(r, \theta^{(t)},\hat w^{(t)}_r(i)\right)}+\frac{1}{T(1-\gamma)}\sum_{t=0}^{T-1}\sum_{i=1}^N\lambda_i^{(t)}\sqrt{E^{\nu^*}\left(c_i, \theta^{(t)},\hat w^{(t)}_{c_i}\right)}\right]\\
&\quad+\beta\eta_\theta V^2(N+1)\left(\frac{N}{2(1-\gamma)^3\xi^2}+\frac{1}{2(1-\gamma)}\right)\\
\end{aligned}
$$
\end{lem}

\begin{proof}
Recall that the virtual sequence $\{\bar\theta^{(t)}\}$ is defined as:
$$
\bar\theta^{(t)}=\frac{1}{N}\sum_{i=1}^N{\theta_i^{(t)}}.
$$
And we always have
$$
\begin{aligned}
\bar\theta^{(t+1)}&=\frac{1}{N}\sum_{i=1}^N\theta_i^{(t+1)}\\
&=\frac{1}{N}\sum_{i=1}^N\left[\theta_i^{(t)}+\eta_\theta\hat w_i^{(t)}\right]\\
&=\frac{1}{N}\sum_{i=1}^N\left[\theta_i^{(t)}+\eta_\theta\left(\hat w_r^{(t)}(i)-N\lambda_i^{(t)}\hat w_{c_i}^{(t)}\right)\right]\\
&=\bar\theta^{(t)}+\eta_\theta\left(\frac{1}{N}\sum_{i=1}^N\hat w_r^{(t)}(i)-\sum_{i=1}^N\lambda_i^{(t)}\hat w_{c_i}^{(t)}\right)\\
&:=\bar\theta^{(t)}+\eta_\theta\hat w^{(t)}
\end{aligned}
$$
We have $\bar\theta^{(t)}=\theta^{(t)}$ when $E\mid t$.
Now we have
\begin{align*}
&\EB_{s\sim d^{\pi^*}}(D_{\operatorname{KL}}(\pi^*(\cdot|s)\|\pi^{[t]}(\cdot|s))-D_{\operatorname{KL}}(\pi^*(\cdot|s)\|\pi^{[t+1]}(\cdot|s)))\\
&=-\EB_{(s,a)\sim\nu^*}\log\frac{\pi^{[t]}(a|s)}{\pi^{[t+1]}(a|s)}\\
&\geq \eta_\theta\EB_{(s,a)\sim\nu^*}[\nabla_{\theta} \log \pi^{[t]}(a | s)^\top \hat w^{(t)}]-\frac{\beta\eta_\theta^2}{2}\|\hat w^{(t)}\|^2_2\\
&= \frac{1}{N}\sum_{i=1}^N\eta_\theta\EB_{(s,a)\sim\nu^*}\left[\nabla_{\theta} \log \pi^{[t]}(a | s)^\top \hat w_i^{(t)}\right]-\frac{\beta\eta_\theta^2}{2}\|\hat w^{(t)}\|^2_2\\
&= \frac{1}{N}\sum_{i=1}^N\eta_\theta\EB_{(s,a)\sim\nu^*}\left[ \left(\nabla_{\theta}\log \pi^{[t]}(a | s)-\nabla_{\theta}\log \pi^{(t)}_i(a | s)\right)^\top \hat w_i^{(t)}\right]+\frac{1}{N}\sum_{i=1}^N\eta_\theta\EB_{(s,a)\sim\nu^*}\left[ \nabla_{\theta}\log \pi^{(t)}_i(a | s)^\top \hat w_i^{(t)}\right]\\
&\quad-\frac{\beta\eta_\theta^2}{2}\|\hat w^{(t)}\|^2_2\\
&\geq -\frac{1}{N}\sum_{i=1}^N\eta_\theta\EB_{(s,a)\sim\nu^*}\left[\left\|\nabla_{\theta}\log \pi^{[t]}(a | s)-\nabla_{\theta}\log \pi^{(t)}_i(a | s)\right\|_2 \left\|\hat w_i^{(t)}\right\|_2\right]+\frac{1}{N}\sum_{i=1}^N\eta_\theta\EB_{(s,a)\sim\nu^*}\left[ \nabla_{\theta}\log \pi^{(t)}_i(a | s)^\top \hat w_i^{(t)}\right]\\
&\quad-\frac{\beta\eta_\theta^2}{2}\|\hat w^{(t)}\|^2_2\\
&=-\frac{1}{N}\sum_{i=1}^N\eta_\theta \beta\left\|\bar\theta^{(t)}-\theta_i^{(t)}\right\|_2\left\|\hat w_i^{(t)}\right\|_2+\frac{1}{N}\sum_{i=1}^N\eta_\theta\EB_{(s,a)\sim\nu^*}\left[ \nabla_{\theta}\log \pi^{(t)}_i(a | s)^\top \hat w_i^{(t)}\right]-\frac{\beta\eta_\theta^2}{2}\|\hat w^{(t)}\|^2_2\\
&=-\frac{1}{N}\sum_{i=1}^N\eta_\theta \beta\left\|\bar\theta^{(t)}-\theta_i^{(t)}\right\|_2\left\|\hat w_i^{(t)}\right\|_2+\frac{1}{N}\sum_{i=1}^N\eta_\theta\EB_{(s,a)\sim\nu^*} A_r^{(t)_i}(s,a)-\sum_{i=1}^N\eta_\theta \EB_{(s,a)\sim\nu^*}\left[\lambda_i^{(t)}A_{c_i}^{(t)_i}(s,a)\right]\\
&\quad+\frac{1}{N}\sum_{i=1}^N\eta_\theta \EB_{(s,a)\sim\nu^*}\left[\nabla_\theta\log \pi^{(t)}_i(a|s)^\top \hat w_i^{(t)}-\left(A_r^{(t)_i}(s,a)-N \lambda_i^{(t)}A_{c_i}^{(t)_i}(s,a)\right)\right]-\frac{\beta\eta_\theta^2}{2}\|\hat w^{(t)}\|^2_2\\
&=-\frac{1}{N}\sum_{i=1}^N\eta_\theta \beta\left\|\bar\theta^{(t)}-\theta_i^{(t)}\right\|_2\left\|\hat w_i^{(t)}\right\|_2+\frac{1}{N}\sum_{i=1}^N\eta_\theta\EB_{(s,a)\sim\nu^*} A_r^{(t)_i}(s,a)-\sum_{i=1}^N\eta_\theta \EB_{(s,a)\sim\nu^*}\left[\lambda_i^{(t)}A_{c_i}^{(t)_i}(s,a)\right]\\
&\quad+\frac{1}{N}\sum_{i=1}^N\eta_\theta \EB_{(s,a)\sim\nu^*}\left[\nabla_\theta\log \pi^{(t)}_i(a|s)^\top \hat w_i^{(t)}-\left(A_r^{(t)_i}(s,a)-N \lambda_i^{(t)}A_{c_i}^{(t)_i}(s,a)\right)\right]-\frac{\beta\eta_\theta^2}{2}\|\hat w^{(t)}\|^2_2\\
&=-\frac{1}{N}\sum_{i=1}^N\eta_\theta \beta\left\|\bar\theta^{(t)}-\theta_i^{(t)}\right\|_2\left\|\hat w_i^{(t)}\right\|_2+\frac{1}{N}\sum_{i=1}^N\eta_\theta\EB_{(s,a)\sim\nu^*} A_r^{(t)_i}(s,a)-\sum_{i=1}^N\eta_\theta \EB_{(s,a)\sim\nu^*}\left[\lambda_i^{(t)}A_{c_i}^{(t)}(s,a)\right]-\frac{\beta\eta_\theta^2}{2}\|\hat w^{(t)}\|^2_2\\
&\quad+\frac{1}{N}\sum_{i=1}^N\eta_\theta \EB_{(s,a)\sim\nu^*}\left[\nabla_\theta\log \pi^{(t)_i}(a|s)^\top \hat w_r^{(t)}(i)-A_r^{(t)_i}(s,a)\right]
-\sum_{i=1}^N \eta_\theta\lambda_i^{(t)}\EB_{(s,a)\sim\nu^*}\left[\nabla_\theta \log\pi^{(t)_i}(a|s)^\top \hat w_{c_i}^{(t)}(t)-A_{c_i}^{(t)_i}(s,a)\right].\\
\end{align*}
Now we apply the performance difference lemma and Jensen's inequality:
$$
\begin{aligned}
&-\frac{1}{N}\sum_{i=1}^N\eta_\theta \beta\left\|\bar\theta^{(t)}-\theta_i^{(t)}\right\|_2\left\|\hat w_i^{(t)}\right\|_2+\frac{1}{N}\sum_{i=1}^N\eta_\theta\EB_{(s,a)\sim\nu^*} A_r^{(t)_i}(s,a)-\sum_{i=1}^N\eta_\theta \EB_{(s,a)\sim\nu^*}\left[\lambda_i^{(t)}A_{c_i}^{(t)}(s,a)\right]-\frac{\beta\eta_\theta^2}{2}\|\hat w^{(t)}\|^2_2\\
&\quad+\frac{1}{N}\sum_{i=1}^N\eta_\theta \EB_{(s,a)\sim\nu^*}\left[\nabla_\theta\log \pi^{(t)_i}(a|s)^\top \hat w_r^{(t)}(i)-A_r^{(t)_i}(s,a)\right]
-\sum_{i=1}^N \eta_\theta\lambda_i^{(t)}\EB_{(s,a)\sim\nu^*}\left[\nabla_\theta \log\pi^{(t)_i}(a|s)^\top \hat w_{c_i}^{(t)}(t)-A_{c_i}^{(t)_i}(s,a)\right]\\
&\geq -\frac{1}{N}\sum_{i=1}^N\eta_\theta \beta\left\|\bar\theta^{(t)}-\theta_i^{(t)}\right\|_2\left\|\hat w_i^{(t)}\right\|_2+\frac{1}{N}\sum_{i=1}^N(1-\gamma)\eta_\theta(V_r^*(\rho)-V_r^{(t)_i}(\rho))- \sum_{i=1}^N(1-\gamma)\eta_\theta\lambda_i^{(t)}(V_{c_i}^{*}(\rho)-V_{c_i}^{(t)_i}(\rho))\\
&\quad-\frac{\eta_\theta}{N}\sum_{i=1}^N\sqrt{E^{\nu^*}\left(r, \theta_i^{(t)},\hat w^{(t)}_r(i)\right)}
-\eta_\theta\sum_{i=1}^N\lambda_i^{(t)}\sqrt{E^{\nu^*}\left(c_i, \theta_i^{(t)},\hat w^{(t)}_{c_i}\right)}-\frac{\beta\eta_\theta^2}{2}\|\hat w^{(t)}\|^2_2.
\end{aligned}
$$
Rearranging terms yields
$$
\begin{aligned}
&\frac{1}{N}\sum_{i=1}^N(V_r^*(\rho)-V_r^{(t)_i}(\rho))\\
&\leq \frac{1}{(1-\gamma)\eta_\theta}\EB_{s\sim d^{\pi^*}}(D_{\operatorname{KL}}(\pi^*(\cdot|s)\|\pi^{[t]}(\cdot|s))-D_{\operatorname{KL}}(\pi^*(\cdot|s)\|\pi^{[t+1]}(\cdot|s)))+\frac{1}{N}\sum_{i=1}^N \frac{\beta}{1-\gamma}\left\|\bar\theta^{(t)}-\theta_i^{(t)}\right\|_2\left\|\hat w_i^{(t)}\right\|_2\\
&\quad+\sum_{i=1}^N\lambda_i^{(t)}(V_{c_i}^{*}(\rho)-V_{c_i}^{(t)_i}(\rho))+\frac{1}{N(1-\gamma)}\sum_{i=1}^N\sqrt{E^{\nu^*}\left(r, \theta_i^{(t)},\hat w^{(t)}_r(i)\right)}+\frac{1}{1-\gamma}\sum_{i=1}^N\lambda_i^{(t)}\sqrt{E^{\nu^*}\left(c_i, \theta_i^{(t)},\hat w^{(t)}_{c_i}\right)}\\
&\quad +\frac{\beta\eta_\theta}{2(1-\gamma)}\|\hat w^{(t)}\|^2_2.
\end{aligned}
$$
Therefore,
$$
\begin{aligned}
&\frac{1}{TN}\sum_{t=0}^{T-1}\sum_{i=1}^N(V_r^*(\rho)-V_r^{(t)_i}(\rho))\\
&\leq \frac{1}{T(1-\gamma)\eta_\theta}\EB_{s\sim d^{\pi^*}}(D_{\operatorname{KL}}(\pi^*(\cdot|s)\|\pi^{[0]}(\cdot|s))-D_{\operatorname{KL}}(\pi^*(\cdot|s)\|\pi^{[T]}(\cdot|s)))+\frac{1}{TN}\sum_{t=0}^{T-1}\sum_{i=1}^N \frac{\beta}{1-\gamma}\left\|\bar\theta^{(t)}-\theta_i^{(t)}\right\|_2\left\|\hat w_i^{(t)}\right\|_2\\
&\quad+\frac{1}{TN(1-\gamma)}\sum_{t=0}^{T-1}\sum_{i=1}^N\sqrt{E^{\nu^*}\left(r, \theta^{(t)},\hat w^{(t)}_r(i)\right)}+\frac{1}{T(1-\gamma)}\sum_{t=0}^{T-1}\sum_{i=1}^N\lambda_i^{(t)}\sqrt{E^{\nu^*}\left(c_i, \theta^{(t)},\hat w^{(t)}_{c_i}\right)}\\
&\quad +\frac{1}{T}\sum_{t=0}^{T-1}\sum_{i=1}^N\lambda_i^{(t)}(V_{c_i}^{*}(\rho)-V_{c_i}^{(t)_i}(\rho))+\frac{\beta\eta_\theta}{2T(1-\gamma)}\sum_{t=0}^{T-1}\|\hat w^{(t)}\|^2_2\\
&\leq \frac{\log|\gA|}{T(1-\gamma)\eta_\theta}+\frac{1}{TN}\sum_{t=0}^{T-1}\sum_{i=1}^N \frac{\beta}{1-\gamma}\left\|\bar\theta^{(t)}-\theta_i^{(t)}\right\|_2\left\|\hat w_i^{(t)}\right\|_2+\frac{1}{T}\sum_{t=0}^{T-1}\sum_{i=1}^N\lambda_i^{(t)}(V_{c_i}^{*}(\rho)-V_{c_i}^{(t)_i}(\rho))+\frac{\beta\eta_\theta}{2T(1-\gamma)}\sum_{t=0}^{T-1}\|\hat w^{(t)}\|^2_2\\
&\quad+\frac{1}{TN(1-\gamma)}\sum_{t=0}^{T-1}\sum_{i=1}^N\sqrt{E^{\nu^*}\left(r, \theta^{(t)},\hat w^{(t)}_r(i)\right)}+\frac{1}{T(1-\gamma)}\sum_{t=0}^{T-1}\sum_{i=1}^N\lambda_i^{(t)}\sqrt{E^{\nu^*}\left(c_i, \theta^{(t)},\hat w^{(t)}_{c_i}\right)}.
\end{aligned}
$$
Now we take expectation to both sides and use $\lambda_i^{(t)}\in[0,\frac{2}{(1-\gamma)\xi}]$:
$$
\begin{aligned}
&\EB\left[\frac{1}{TN}\sum_{t=0}^{T-1}\sum_{i=1}^N(V_r^*(\rho)-V_r^{(t)_i}(\rho))\right]\\
&\leq \frac{\log|\gA|}{T(1-\gamma)\eta_\theta}+\EB\left[\frac{1}{TN}\sum_{t=0}^{T-1}\sum_{i=1}^N \frac{\beta}{1-\gamma}\left\|\bar\theta^{(t)}-\theta_i^{(t)}\right\|_2\left\|\hat w_i^{(t)}\right\|_2\right]+\EB\left[\frac{1}{T}\sum_{t=0}^{T-1}\sum_{i=1}^N\lambda_i^{(t)}(V_{c_i}^{*}(\rho)-V_{c_i}^{(t)_i}(\rho))\right]\\
&\quad+\EB\left[\frac{1}{TN(1-\gamma)}\sum_{t=0}^{T-1}\sum_{i=1}^N\sqrt{E^{\nu^*}\left(r, \theta^{(t)},\hat w^{(t)}_r(i)\right)}+\frac{1}{T(1-\gamma)}\sum_{t=0}^{T-1}\sum_{i=1}^N\lambda_i^{(t)}\sqrt{E^{\nu^*}\left(c_i, \theta^{(t)},\hat w^{(t)}_{c_i}\right)}\right]\\
&\quad+\EB\left[\frac{\beta\eta_\theta}{2T(1-\gamma)}\sum_{t=0}^{T-1}\|\hat w^{(t)}\|^2_2\right]\\
&\leq \frac{\log|\gA|}{T(1-\gamma)\eta_\theta}+\EB\left[\frac{1}{TN}\sum_{t=0}^{T-1}\sum_{i=1}^N \frac{\beta}{1-\gamma}\left\|\bar\theta^{(t)}-\theta_i^{(t)}\right\|_2\left\|\hat w_i^{(t)}\right\|_2\right]+\EB\left[\frac{1}{T}\sum_{t=0}^{T-1}\sum_{i=1}^N\lambda_i^{(t)}(V_{c_i}^{*}(\rho)-V_{c_i}^{(t)_i}(\rho))\right]\\
&\quad+\EB\left[\frac{1}{TN(1-\gamma)}\sum_{t=0}^{T-1}\sum_{i=1}^N\sqrt{E^{\nu^*}\left(r, \theta^{(t)},\hat w^{(t)}_r(i)\right)}+\frac{1}{T(1-\gamma)}\sum_{t=0}^{T-1}\sum_{i=1}^N\lambda_i^{(t)}\sqrt{E^{\nu^*}\left(c_i, \theta^{(t)},\hat w^{(t)}_{c_i}\right)}\right]\\
&\quad+\beta\eta_\theta V^2(N+1)\left(\frac{N}{2(1-\gamma)^3\xi^2}+\frac{1}{2(1-\gamma)}\right).
\end{aligned}
$$
We complete the proof.
\end{proof}

\begin{lem}\label{Lemma_bound_bar_theta_and_theta_i}
For any $i\in\{1,...,N\}$, $t\in\{0,...,T-1\}$, 
$$
\begin{aligned}
\EB \|\bar\theta^{(t)}-\theta_i^{(t)}\|_2^2&\leq 4(E-1)^2\eta_\theta^2V^2(N+1)\left(\frac{N}{(1-\gamma)^2\xi^2}+1\right)\\
\EB \|\theta^{(t)}-\theta_i^{(t)}\|_2&\leq (E-1)\eta_\theta V\left(\frac{N}{(1-\gamma)\xi}+1\right)\\
\end{aligned}
$$
\end{lem}
\begin{proof}
By the definition of $\bar\theta^{(t)}$, we may always find $t-E<t_0\leq t$ such that $\bar\theta_i^{(t_0)}=\theta_i^{(t_0)}$, $\forall i \in \{1,...,N\}$.
If $t_0=t$, then the conclusion holds trivially.
Else we have
$$
\begin{aligned}
&\EB \|\bar\theta^{(t)}-\theta_i^{(t)}\|_2^2\\
&\leq2\EB \|\bar\theta^{(t)}-\bar\theta^{(t_0)}\|_2^2+2\EB \|\theta_i^{(t)}-\theta_i^{(t_0)}\|_2^2\\
&=2\EB\left[\left\|\sum_{t^\prime=t_0}^{t-1}\eta_\theta\hat w^{(t^\prime)}\right\|_{2}^2\right]+2\EB\left[\left\|\sum_{t^\prime=t_0}^{t-1}\eta_\theta\hat w_i^{(t^\prime)}\right\|_{2}^2\right]\\
&\leq 2(t-1-t_0)\sum_{t^\prime=t_0}^{t-1}\EB\left[\left\|\eta_\theta\hat w^{(t^\prime)}\right\|_{2}^2\right]+2(t-1-t_0)\sum_{t^\prime=t_0}^{t-1}\EB\left[\left\|\eta_\theta\hat w_i^{(t^\prime)}\right\|_{2}^2\right]\\
&\leq 2(E-1)\sum_{t^\prime=t_0}^{t_0+E-2}\EB\left[\left\|\eta_\theta\hat w^{(t^\prime)}\right\|_{2}^2\right]+2(E-1)\sum_{t^\prime=t_0}^{t_0+E-2}\EB\left[\left\|\eta_\theta\hat w_i^{(t^\prime)}\right\|_{2}^2\right]\\
&\leq 4(E-1)^2\eta_\theta^2V^2(N+1)\left(\frac{N}{(1-\gamma)^2\xi^2}+1\right)\\
\end{aligned}
$$
Similarly, there always exists $t-E<t_0\leq t$ such that $\theta^{(t)}=\theta^{(t_0)}=\theta_i^{(t_0)}$, $\forall i \in \{1,...,N\}$. 
If $t_0=t$, then the conclusion holds trivially. 
Else we have
$$
\begin{aligned}
&\EB \|\theta^{(t)}-\theta_i^{(t)}\|_2\\
&=\EB \|\theta_i^{(t)}-\theta_i^{(t_0)}\|_2\\
&=\EB\left[\left\|\sum_{t^\prime=t_0}^{t-1}\eta_\theta\hat w_i^{(t^\prime)}\right\|_{2}\right]\\
&\leq \sum_{t^\prime=t_0}^{t-1}\EB\left[\left\|\eta_\theta\hat w_i^{(t^\prime)}\right\|_{2}\right]\\
&\leq \sum_{t^\prime=t_0}^{t_0+E-2}\EB\left[\left\|\eta_\theta\hat w_i^{(t^\prime)}\right\|_{2}\right]\\
&\leq (E-1)\eta_\theta V\left(\frac{N}{(1-\gamma)\xi}+1\right)\\
\end{aligned}
$$
We complete the proof.
\end{proof}

\begin{lem}\label{Lemma_SGD_convergence}
For any $i\in\{1,...,N\}$, we have
$$
\begin{aligned}
\EB E^{\nu^*}(c_i,\theta^{(t)}, \hat w_{c_i}^{(t)})&\leq \frac{1}{1-\gamma}\left\|\frac{\nu^*}{\nu_0}\right\|_\infty\left(\epsilon_{bias}+\frac{2\left(2\sqrt{d}WL_\pi+\frac{2\sqrt{d}}{1-\gamma}+WL_\pi\right)^2}{K}\right),\\
\EB E^{\nu^*}(r,\theta^{(t)}, \hat w_{r}^{(t)}(i))&\leq \frac{1}{1-\gamma}\left\|\frac{\nu^*}{\nu_0}\right\|_\infty\left(\epsilon_{bias}+\frac{2\left(2\sqrt{d}WL_\pi+\frac{2\sqrt{d}}{1-\gamma}+WL_\pi\right)^2}{K}\right),
\end{aligned}
$$
where $\nu_0$ is the uniform distribution on $\gS\times\gA$.
\end{lem}

\begin{proof}
Here we show 
$$
\EB E^{\nu^*}(c_1,\theta^{(t)}, \hat w_{c_1}^{(t)})\leq \frac{1}{1-\gamma}\left\|\frac{\nu^*}{\nu_0}\right\|_\infty\left(\epsilon_{bias}+\frac{2\left(2\sqrt{d}WL_\pi+\frac{2\sqrt{d}}{1-\gamma}+WL_\pi\right)^2}{K}\right),
$$
and the full conclusion can be obtained via similar arguments.
First we have:
$$
\begin{aligned}
&\EB\left[E^{\nu^*}(c_1,\theta^{(t)}, \hat w_{c_1}^{(t)})\right]\\
&=\EB\left[E^{\nu^*}(c_1,\theta^{(t)}, \hat w_{c_1}^{(t)})\right]\\
&\leq \EB\left[\left\|\frac{\nu^{*}}{\nu^{(t)}}\right\|_{\infty} E^{\nu^{(t)}}(c_1,\theta^{(t)}, \hat w_{c_1}^{(t)})\right]\\
&\leq \frac{1}{1-\gamma} \EB\left[\left\|\frac{\nu^{*}}{\nu_{0}}\right\|_{\infty} E^{\nu^{(t)}}(c_1,\theta^{(t)}, \hat w_{c_1}^{(t)})\right]\\
&=\frac{1}{1-\gamma}\left\|\frac{\nu^{*}}{\nu_{0}}\right\|_{\infty}\left(\EB \min_w E^{\nu^{(t)}}(c_1,\theta^{(t)},w)+\EB\left[E^{\nu^{(t)}}\left(c_1,\theta^{(t)}, \hat w^{(t)}_{c_1}\right)-\min_w E^{\nu^{(t)}}(c_1,\theta^{(t)},w)\right]\right)\\
&\leq \frac{1}{1-\gamma}\left\|\frac{\nu^{*}}{\nu_{0}}\right\|_{\infty}\left(\epsilon_{bias}+\EB\left[E^{\nu^{(t)}}\left(c_1,\theta^{(t)}, \hat w^{(t)}_{c_1}\right)-\min_w E^{\nu^{(t)}}(c_1,\theta^{(t)},w)\right]\right) .
\end{aligned}
$$
Set $\alpha=\frac{1}{4L_\pi^2}$, by Theorem 1 in \cite{bach2013non} we may get:
$$
\EB\left[E^{\nu^{(t)}}\left(c_1,\theta^{(t)}, \hat w^{(t)}_{c_1}\right)-\min_w E^{\nu^{(t)}}(c_1,\theta^{(t)},w)\right]\leq \frac{2(\sigma\sqrt{d}+L_\pi W)^2}{K}.
$$
$\sigma$ is defined as:
$$
\EB_{(s,a)\sim\nu^{(t)}}\left[G_{c_1}^{(t)}\left(G_{c_1}^{(t)}\right)^\top\right]\preceq \sigma^2 F(\theta^{(t)}),
$$
where $G_{c_1}^{(t)}:=2\left(\left(w_{c_{1}}^{*}\right)^\top \nabla_{\theta} \log \pi^{(t)}(a \mid s)-\widehat{A}_{c_{1}}^{(t)}(s, a)\right) \nabla_{\theta} \log \pi^{(t)}(a \mid s)$ and $w_{c_1}^*:=\underset{w}{\mathrm{argmin}} E^{\nu^{(t)}}(c_i,\theta^{(t)},w).$
We have $\sigma\leq 2WL_\pi +\frac{2}{1-\gamma}$. The proof is completed.
\end{proof}

\begin{lem}\label{Lemma_bound_on_var}
For $\forall i\in\{1,...,N\}$,
$$
\EB\left[\sum_{t=0}^{T-1}\left(d_{i}-\widehat{V}_{c_{i}}^{(t)_i}(\rho)\right)^2\right]\leq \frac{3T}{(1-\gamma)^2}.
$$
\end{lem}
\begin{proof}
Note that:
$$
\begin{aligned}
\EB\left[\sum_{t=0}^{T-1}\left(d_{i}-\widehat{V}_{c_{i}}^{(t)_i}(\rho)\right)^2\right]&=\sum_{t=0}^{T-1}d_i^2-2\sum_{t=0}^{T-1}d_i\EB\widehat{V}_{c_{i}}^{(t)_i}(\rho)+\sum_{t=0}^{T-1}\EB\left[\widehat{V}_{c_{i}}^{(t)_i}(\rho)\right]^2\\
&\leq \frac{T}{(1-\gamma)^2}+\sum_{t=0}^{T-1}\EB\left[\widehat{V}_{c_{i}}^{(t)_i}(\rho)\right]^2.
\end{aligned}
$$
Then it suffices to bound $\EB\left[\widehat{V}_{c_{i}}^{(t)_i}(\rho)\right]^2$.
Since
$$
\begin{aligned}
\EB\left[\widehat{V}_{c_{i}}^{(t)_i}(\rho)\right]^2 &=\operatorname{Var}\left[\widehat{V}_{c_{i}}^{(t)_i}(\rho)\right]+\left[\EB\widehat{V}_{c_{i}}^{(t)_i}(\rho)\right]^2\\
&=\frac{1}{K}\operatorname{Var}\left[\widehat{V}_{c_{i}}^{(t)_i}(s)\right]+\left[\EB\widehat{V}_{c_{i}}^{(t)_i}(\rho)\right]^2\\
&=\frac{1}{K}\EB\left[\widehat{V}_{c_{i}}^{(t)_i}(s)-{V}_{c_{i}}^{(t)_i}(s)\right]^2+\left[\EB\widehat{V}_{c_{i}}^{(t)_i}(\rho)\right]^2\\
&=\frac{1}{K}\EB_{K^\prime}\EB\left[\left(\sum_{k=1}^{K^\prime}c_i(s_k,a_k)-{V}_{c_{i}}^{(t)_i}(s)\right)^2\Bigg| K^\prime\right]+\left[\EB\widehat{V}_{c_{i}}^{(t)_i}(\rho)\right]^2\\
&\leq\frac{1}{K}\EB[K^\prime]^2+\left[\EB\widehat{V}_{c_{i}}^{(t)_i}(\rho)\right]^2\\
&\leq \frac{1+K}{K(1-\gamma)^2}\\
&\leq \frac{2}{(1-\gamma)^2}\\
\end{aligned}
$$
we complete the proof.
\end{proof}

\begin{lem}\label{Lemma_constraint_violation}
We have
$$
\EB\left[\frac{1}{T} \sum_{t=0}^{T-1} \sum_{i=1}^{N} \lambda_{i}^{(t)}\left(V_{c_{i}}^{*}(\rho)-V_{c_{i}}^{(t)_i}(\rho)\right)\right]\leq
\frac{2N\eta_\lambda}{(1-\gamma)^2}.
$$
\end{lem}
\begin{proof}
Note that $\forall i\in\{1,...,N\}$,
$$
\begin{aligned}
    \left(\lambda_i^{(T)}\right)^2&=\sum_{t=0}^{T-1}\left[\left(\lambda_i^{(t+1)}\right)^2-\left(\lambda_i^{(t)}\right)^2\right]\\
    &=\sum_{t=0}^{T-1}\left[\left(\operatorname{Proj}_\Lambda\left(\lambda_i^{(t)}-\eta_\lambda\left(d_i-\widehat{V}_{c_{i}}^{(t)_i}(\rho)\right)\right)\right)^2-\left(\lambda_i^{(t)}\right)^2\right]\\
    &\leq \sum_{t=0}^{T-1}\left[\left(\lambda_i^{(t)}-\eta_\lambda\left(d_i-\widehat{V}_{c_{i}}^{(t)_i}(\rho)\right)\right)^2-\left(\lambda_i^{(t)}\right)^2\right]\\
    &=\sum_{t=0}^{T-1}\eta_\lambda^2\left(d_i-\widehat{V}_{c_{i}}^{(t)_i}(\rho)\right)^2+2\sum_{t=0}^{T-1}\eta_\lambda\lambda_i^{(t)}\left(\widehat{V}_{c_{i}}^{(t)_i}(\rho)-d_i\right)\\
    &=\sum_{t=0}^{T-1}\eta_\lambda^2\left(d_i-\widehat{V}_{c_{i}}^{(t)_i}(\rho)\right)^2+2\sum_{t=0}^{T-1}\eta_\lambda\lambda_i^{(t)}\left(\widehat{V}_{c_{i}}^{(t)_i}(\rho)-{V}_{c_{i}}^{(t)_i}(\rho)\right)+2\sum_{t=0}^{T-1}\eta_\lambda\lambda_i^{(t)}\left({V}_{c_{i}}^{(t)_i}(\rho)-d_i\right)\\
    &\leq\sum_{t=0}^{T-1}\eta_\lambda^2\left(d_i-\widehat{V}_{c_{i}}^{(t)_i}(\rho)\right)^2+2\sum_{t=0}^{T-1}\eta_\lambda\lambda_i^{(t)}\left(\widehat{V}_{c_{i}}^{(t)_i}(\rho)-{V}_{c_{i}}^{(t)_i}(\rho)\right)+2\sum_{t=0}^{T-1}\eta_\lambda\lambda_i^{(t)}\left({V}_{c_{i}}^{(t)_i}(\rho)-V^*_{c_i}(\rho)\right).\\
    &
\end{aligned}
$$
The last inequality holds because of the feasibility of the optimal policy $\pi^*$.
We take expectations of both sides and rearrange terms:
$$
\EB\left[\sum_{t=0}^{T-1}\eta_\lambda\lambda_i^{(t)}\left(V^*_{c_i}(\rho)-{V}_{c_{i}}^{(t)_i}(\rho)\right)\right]\leq \frac{1}{2}\EB\left[\sum_{t=0}^{T-1}\eta_\lambda^2\left(d_{i}-\widehat{V}_{c_{i}}^{(t)_i}(\rho)\right)^2\right]+\EB\left[\sum_{t=0}^{T-1}\eta_\lambda\lambda_i^{(t)}\left(\widehat{V}_{c_{i}}^{(t)_i}(\rho)-{V}_{c_{i}}^{(t)_i}(\rho)\right)\right].
$$
First, we have $\EB\left[\sum_{t=0}^{T-1}\eta_\lambda\lambda_i^{(t)}\left(\widehat{V}_{c_{i}}^{(t)_i}(\rho)-{V}_{c_{i}}^{(t)_i}(\rho)\right)\right]=0$ due to the conditional unbiasedness of $\widehat{V}_{c_{i}}^{(t)_i}(\rho)$.
Using Lemma~\ref{Lemma_bound_on_var}, we also have 
$$
\EB\left[\sum_{t=0}^{T-1}\eta_\lambda^2\left(d_{i}-\widehat{V}_{c_{i}}^{(t)_i}(\rho)\right)^2\right]\leq \frac{3T\eta_\lambda^2}{(1-\gamma)^2}.
$$
Thus
$$\begin{aligned}
    \EB\left[\frac{1}{T} \sum_{t=0}^{T-1} \sum_{i=1}^{N} \lambda_{i}^{(t)}\left(V_{c_{i}}^{*}(\rho)-V_{c_{i}}^{(t)_i}(\rho)\right)\right]&\leq \frac{1}{2}\EB\left[\frac{1}{T} \sum_{t=0}^{T-1} \sum_{i=1}^{N}\eta_\lambda\left(d_{i}-\widehat{V}_{c_{i}}^{(t)_i}(\rho)\right)^2 \right]\\
&\leq \frac{2N\eta_\lambda}{(1-\gamma)^2}.
\end{aligned}
$$
We complete the proof.
\end{proof}
\begin{lem}\label{Lemma_bound_of_reward}
We have
$$
\begin{aligned}
&\EB\left[\frac{1}{T}\sum_{t=0}^{T-1}(V_r^*(\rho)-V_r^{(t)}(\rho))\right]\\
&\leq \frac{\log|\gA|}{T(1-\gamma)\eta_\theta}+\frac{(2(E-1)+1/2)\beta\eta_\theta V^2(N+1)}{1-\gamma}\left(\frac{N}{(1-\gamma)^2\xi^2}+1\right)+|\gA|L_\pi (E-1)\eta_\theta V\left(\frac{N}{(1-\gamma)^3\xi}+\frac{1}{(1-\gamma)^2}\right)\\
&\quad+\left(\frac{1}{1-\gamma}+\frac{2}{(1-\gamma)^2\xi}\right)\sqrt{\frac{1}{1-\gamma}\left\|\frac{\nu^*}{\nu_0}\right\|_\infty\left(\epsilon_{bias}+\frac{2\left(2\sqrt{d}VL_\pi+\frac{2\sqrt{d}}{1-\gamma}+VL_\pi\right)^2}{K}\right)}.\\
\end{aligned}
$$
\end{lem}
\begin{proof}
$$
\begin{aligned}
&\EB\left[\frac{1}{TN}\sum_{t=0}^{T-1}\sum_{i=1}^N(V_r^*(\rho)-V_r^{(t)_i}(\rho))\right]\\
&\leq \frac{\log|\gA|}{T(1-\gamma)\eta_\theta}+\EB\left[\frac{1}{TN}\sum_{t=0}^{T-1}\sum_{i=1}^N \frac{\beta}{1-\gamma}\left\|\bar\theta^{(t)}-\theta_i^{(t)}\right\|_2\left\|\hat w_i^{(t)}\right\|_2\right]+\EB\left[\frac{1}{T}\sum_{t=0}^{T-1}\sum_{i=1}^N\lambda_i^{(t)}(V_{c_i}^{*}(\rho)-V_{c_i}^{(t)_i}(\rho))\right]\\
&\quad+\EB\left[\frac{1}{TN(1-\gamma)}\sum_{t=0}^{T-1}\sum_{i=1}^N\sqrt{E^{\nu^*}\left(r, \theta^{(t)},\hat w^{(t)}_r(i)\right)}+\frac{1}{T(1-\gamma)}\sum_{t=0}^{T-1}\sum_{i=1}^N\lambda_i^{(t)}\sqrt{E^{\nu^*}\left(c_i, \theta^{(t)},\hat w^{(t)}_{c_i}\right)}\right]\\
&\quad+\beta\eta_\theta V^2(N+1)\left(\frac{N}{2(1-\gamma)^3\xi^2}+\frac{1}{2(1-\gamma)}\right)\\
&\leq \frac{\log|\gA|}{T(1-\gamma)\eta_\theta}+\EB\left[\frac{1}{TN}\sum_{t=0}^{T-1}\sum_{i=1}^N \frac{\beta}{1-\gamma}\left\|\bar\theta^{(t)}-\theta_i^{(t)}\right\|_2\left\|\hat w_i^{(t)}\right\|_2\right]+\frac{2N\eta_\lambda}{(1-\gamma)^2}\\
&\quad+\left(\frac{1}{1-\gamma}+\frac{2}{(1-\gamma)^2\xi}\right)\sqrt{\frac{1}{1-\gamma}\left\|\frac{\nu^*}{\nu_0}\right\|_\infty\left(\epsilon_{bias}+\frac{2\left(2\sqrt{d}VL_\pi+\frac{2\sqrt{d}}{1-\gamma}+VL_\pi\right)^2}{K}\right)}\\
&\quad+\beta\eta_\theta V^2(N+1)\left(\frac{N}{2(1-\gamma)^3\xi^2}+\frac{1}{2(1-\gamma)}\right)\\
\end{aligned}
$$
The first inequality is true due to Lemma~\ref{Lemma_basic_expansion} and the second inequality is true due to Lemma~\ref{Lemma_SGD_convergence} and Lemma~\ref{Lemma_constraint_violation}.
By Cauchy's inequality and Lemma~\ref{Lemma_bound_bar_theta_and_theta_i} we may get
$$
\begin{aligned}
&\EB\left[\frac{1}{TN}\sum_{t=0}^{T-1}\sum_{i=1}^N \frac{\beta}{1-\gamma}\left\|\bar\theta^{(t)}-\theta_i^{(t)}\right\|_2\left\|\hat w_i^{(t)}\right\|_2\right]\\
&\leq \frac{1}{TN}\sum_{t=0}^{T-1}\sum_{i=1}^N \frac{\beta}{1-\gamma}\sqrt{\EB\left\|\bar\theta^{(t)}-\theta_i^{(t)}\right\|_2^2}\sqrt{\EB\left\|\hat w_i^{(t)}\right\|_2^2}\\
&\leq \frac{1}{TN}\sum_{t=0}^{T-1}\sum_{i=1}^N \frac{\beta}{1-\gamma}\left(2(E-1)\eta_\theta V^2(N+1)\left(\frac{N}{(1-\gamma)^2\xi^2}+1\right)\right)\\
&=\frac{2(E-1)\beta\eta_\theta V^2(N+1)}{1-\gamma}\left(\frac{N}{(1-\gamma)^2\xi^2}+1\right)
\end{aligned}
$$
Thus we have
$$
\begin{aligned}
&\EB\left[\frac{1}{TN}\sum_{t=0}^{T-1}\sum_{i=1}^N(V_r^*(\rho)-V_r^{(t)_i}(\rho))\right]\\
&\leq \frac{\log|\gA|}{T(1-\gamma)\eta_\theta}+\frac{(2(E-1)+1/2)\beta\eta_\theta V^2(N+1)}{1-\gamma}\left(\frac{N}{(1-\gamma)^2\xi^2}+1\right)\\
&\quad+\left(\frac{1}{1-\gamma}+\frac{2}{(1-\gamma)^2\xi}\right)\sqrt{\frac{1}{1-\gamma}\left\|\frac{\nu^*}{\nu_0}\right\|_\infty\left(\epsilon_{bias}+\frac{2\left(2\sqrt{d}VL_\pi+\frac{2\sqrt{d}}{1-\gamma}+VL_\pi\right)^2}{K}\right)}.\\
\end{aligned}
$$
Finally, we may complete the proof by noting that by Lemma~\ref{Lemma_Lipschitz_value} and Lemma~\ref{Lemma_bound_bar_theta_and_theta_i}, $\forall t\in\{0,...,T-1\}$, $\forall i\in\{1,...,N\}$
$$
\begin{aligned}
\EB|V_r^{(t)_i}-V_r^{(t)}|&\leq \frac{|\gA|L_\pi\EB\|\theta^{(t)}-\theta_i^{(t)}\|_2}{(1-\gamma)^2}\\
&\leq |\gA|L_\pi (E-1)\eta_\theta V\left(\frac{N}{(1-\gamma)^3\xi}+\frac{1}{(1-\gamma)^2}\right).
\end{aligned}
$$
Therefore, 
$$
\begin{aligned}
&\EB\left[\frac{1}{T}\sum_{t=0}^{T-1}(V_r^*(\rho)-V_r^{(t)}(\rho))\right]\\
&\leq\EB\left[\frac{1}{TN}\sum_{t=0}^{T-1}\sum_{i=1}^N(V_r^*(\rho)-V_r^{(t)_i}(\rho))\right]+\EB\left[\frac{1}{T}\sum_{t=0}^{T-1}|V_r^{(t)}(\rho)-V_r^{(t)_i}(\rho)|\right]\\
&\leq \frac{\log|\gA|}{T(1-\gamma)\eta_\theta}+\frac{(2(E-1)+1/2)\beta\eta_\theta V^2(N+1)}{1-\gamma}\left(\frac{N}{(1-\gamma)^2\xi^2}+1\right)+|\gA|L_\pi (E-1)\eta_\theta V\left(\frac{N}{(1-\gamma)^3\xi}+\frac{1}{(1-\gamma)^2}\right)\\
&\quad+\left(\frac{1}{1-\gamma}+\frac{2}{(1-\gamma)^2\xi}\right)\sqrt{\frac{1}{1-\gamma}\left\|\frac{\nu^*}{\nu_0}\right\|_\infty\left(\epsilon_{bias}+\frac{2\left(2\sqrt{d}VL_\pi+\frac{2\sqrt{d}}{1-\gamma}+VL_\pi\right)^2}{K}\right)}.\\
\end{aligned}
$$
\end{proof}

\begin{lem}\label{Lemma_bound_on_constraint}
$$
\begin{aligned}
&\sum_{i=1}^N\EB\left[\frac{1}{T}\sum_{t=0}^{T-1}V_{c_i}^{(t)}(\rho)-d_i\right]_+\\
&\leq(1-\gamma)\xi \Bigg\{\frac{\log|\gA|}{T(1-\gamma)\eta_\theta}+\frac{(2(E-1)+1/2)\beta\eta_\theta V^2(N+1)}{1-\gamma}\left(\frac{N}{(1-\gamma)^2\xi^2}+1\right)+\frac{2N}{T\eta_\lambda(1-\gamma)^2\xi^2}+\frac{2N\eta_\lambda}{(1-\gamma)^2}\\
&\quad+\left(\frac{1}{1-\gamma}+\frac{2}{(1-\gamma)^2\xi}\right)\sqrt{\frac{1}{1-\gamma}\left\|\frac{\nu^*}{\nu_0}\right\|_\infty\left(\epsilon_{bias}+\frac{2\left(2\sqrt{d}VL_\pi+\frac{2\sqrt{d}}{1-\gamma}+VL_\pi\right)^2}{K}\right)}\\
&\quad +\frac{|\gA|L_\pi (E-1)\eta_\theta V}{(1-\gamma)^2}\left(\frac{N}{(1-\gamma)\xi}+1\right)^2\Bigg\}.\\
\end{aligned}
$$
\end{lem}
\begin{proof}
For any $\lambda=(\lambda_1,...,\lambda_N)\in \Lambda^N$, 
$$
\begin{aligned}
&\EB\left[\frac{1}{TN}\sum_{t=0}^{T-1}\sum_{i=1}^N(V_r^*(\rho)-V_r^{(t)_i}(\rho))\right]+\EB\left[\frac{1}{T}\sum_{t=0}^{T-1}\sum_{i=1}^N\lambda_i(V_{c_i}^{(t)_i}(\rho)-d_i)\right]\\
&\leq \frac{\log|\gA|}{T(1-\gamma)\eta_\theta}+\frac{(2(E-1)+1/2)\beta\eta_\theta V^2(N+1)}{1-\gamma}\left(\frac{N}{(1-\gamma)^2\xi^2}+1\right)+\frac{\sum_{i=1}^N \lambda_i^2}{2T\eta_\lambda}+\frac{2N\eta_\lambda}{(1-\gamma)^2}\\
&\quad+\left(\frac{1}{1-\gamma}+\frac{2}{(1-\gamma)^2\xi}\right)\sqrt{\frac{1}{1-\gamma}\left\|\frac{\nu^*}{\nu_0}\right\|_\infty\left(\epsilon_{bias}+\frac{2\left(2\sqrt{d}VL_\pi+\frac{2\sqrt{d}}{1-\gamma}+VL_\pi\right)^2}{K}\right)}.\\
\end{aligned}
$$
By Lemma~\ref{Lemma_Lipschitz_value} and Lemma~\ref{Lemma_bound_bar_theta_and_theta_i}, we have
$$
\begin{aligned}
\EB|V_\diamond^{(t)_i}-V_\diamond^{(t)}|&\leq \frac{|\gA|L_\pi\EB\|\theta^{(t)}-\theta_i^{(t)}\|_2}{(1-\gamma)^2}\\
&\leq |\gA|L_\pi (E-1)\eta_\theta V\left(\frac{N}{(1-\gamma)^3\xi}+\frac{1}{(1-\gamma)^2}\right).
\end{aligned}
$$
Thus
$$
\begin{aligned}
&\EB\left[\frac{1}{T}\sum_{t=0}^{T-1}(V_r^*(\rho)-V_r^{(t)}(\rho))\right]+\EB\left[\frac{1}{T}\sum_{t=0}^{T-1}\sum_{i=1}^N\lambda_i(V_{c_i}^{(t)}(\rho)-d_i)\right]\\
&\leq \EB\left[\frac{1}{TN}\sum_{t=0}^{T-1}\sum_{i=1}^N(V_r^*(\rho)-V_r^{(t)_i}(\rho))\right]+\EB\left[\frac{1}{TN}\sum_{t=0}^{T-1}\sum_{i=1}^N|V_r^{(t)}(\rho)-V_r^{(t)_i}(\rho)|\right]\\
&\quad+ \EB\left[\frac{1}{T}\sum_{t=0}^{T-1}\sum_{i=1}^N\lambda_i(V_{c_i}^{(t)_i}(\rho)-d_i)\right]+\EB\left[\frac{1}{T}\sum_{t=0}^{T-1}\sum_{i=1}^N\lambda_i|V_{c_i}^{(t)_i}(\rho)-V_{c_i}^{(t)}(\rho)|\right]\\
&\leq \frac{\log|\gA|}{T(1-\gamma)\eta_\theta}+\frac{(2(E-1)+1/2)\beta\eta_\theta V^2(N+1)}{1-\gamma}\left(\frac{N}{(1-\gamma)^2\xi^2}+1\right)+\frac{\sum_{i=1}^N \lambda_i^2}{2T\eta_\lambda}+\frac{2N\eta_\lambda}{(1-\gamma)^2}\\
&\quad+\left(\frac{1}{1-\gamma}+\frac{2}{(1-\gamma)^2\xi}\right)\sqrt{\frac{1}{1-\gamma}\left\|\frac{\nu^*}{\nu_0}\right\|_\infty\left(\epsilon_{bias}+\frac{2\left(2\sqrt{d}VL_\pi+\frac{2\sqrt{d}}{1-\gamma}+VL_\pi\right)^2}{K}\right)}\\
&\quad + \left(1+\sum_{i=1}^N\lambda_i\right)|\gA|L_\pi (E-1)\eta_\theta V\left(\frac{N}{(1-\gamma)^3\xi}+\frac{1}{(1-\gamma)^2}\right)
.\\
\end{aligned}
$$
We take $\lambda_i=0$ if $\frac{1}{T}\sum_{t=0}^{T-1}V_{c_i}^{(t)_i}(\rho)\leq d_i$ and $\lambda_i=\frac{2}{(1-\gamma)\xi}$ otherwise.
Then we obtain
$$
\begin{aligned}
&\EB\left[\frac{1}{T}\sum_{t=0}^{T-1}(V_r^*(\rho)-V_r^{(t)}(\rho))\right]+\frac{2}{(1-\gamma)\xi}\sum_{i=1}^N\EB\left[\frac{1}{T}\sum_{t=0}^{T-1}V_{c_i}^{(t)}(\rho)-d_i\right]_+\\
&\leq \frac{\log|\gA|}{T(1-\gamma)\eta_\theta}+\frac{(2(E-1)+1/2)\beta\eta_\theta V^2(N+1)}{1-\gamma}\left(\frac{N}{(1-\gamma)^2\xi^2}+1\right)+\frac{2N}{T\eta_\lambda(1-\gamma)^2\xi^2}+\frac{2N\eta_\lambda}{(1-\gamma)^2}\\
&\quad+\left(\frac{1}{1-\gamma}+\frac{2}{(1-\gamma)^2\xi}\right)\sqrt{\frac{1}{1-\gamma}\left\|\frac{\nu^*}{\nu_0}\right\|_\infty\left(\epsilon_{bias}+\frac{2\left(2\sqrt{d}VL_\pi+\frac{2\sqrt{d}}{1-\gamma}+VL_\pi\right)^2}{K}\right)}\\
&\quad +\frac{|\gA|L_\pi (E-1)\eta_\theta V}{(1-\gamma)^2}\left(\frac{N}{(1-\gamma)\xi}+1\right)^2
.\\
\end{aligned}
$$
Noting that there always exits a policy $\pi^\prime$ such that $V_\diamond^{\pi^\prime}(\rho)=\frac{1}{T}\sum_{t=0}^{T-1}V_\diamond^{(t)}(\rho)$:
$$
\begin{aligned}
&\EB\left[(V_r^*(\rho)-V_r^{\pi^\prime}(\rho))\right]+\frac{2}{(1-\gamma)\xi}\sum_{i=1}^N\EB\left[V_{c_i}^{\pi^\prime}(\rho)-d_i\right]_+\\
&\leq \frac{\log|\gA|}{T(1-\gamma)\eta_\theta}+\frac{(2(E-1)+1/2)\beta\eta_\theta V^2(N+1)}{1-\gamma}\left(\frac{N}{(1-\gamma)^2\xi^2}+1\right)+\frac{2N}{T\eta_\lambda(1-\gamma)^2\xi^2}+\frac{2N\eta_\lambda}{(1-\gamma)^2}\\
&\quad+\left(\frac{1}{1-\gamma}+\frac{2}{(1-\gamma)^2\xi}\right)\sqrt{\frac{1}{1-\gamma}\left\|\frac{\nu^*}{\nu_0}\right\|_\infty\left(\epsilon_{bias}+\frac{2\left(2\sqrt{d}VL_\pi+\frac{2\sqrt{d}}{1-\gamma}+VL_\pi\right)^2}{K}\right)}\\
&\quad +\frac{|\gA|L_\pi (E-1)\eta_\theta V}{(1-\gamma)^2}\left(\frac{N}{(1-\gamma)\xi}+1\right)^2.
\end{aligned}
$$
Now we apply Lemma~\ref{Lemma_optimization_constraint} to get the conclusion:
$$
\begin{aligned}
&\sum_{i=1}^N\EB\left[\frac{1}{T}\sum_{t=0}^{T-1}V_{c_i}^{(t)}(\rho)-d_i\right]_+=\sum_{i=1}^N\EB\left[V_{c_i}^{\pi^\prime}(\rho)-d_i\right]_+\\
&\leq(1-\gamma)\xi \Bigg\{\frac{\log|\gA|}{T(1-\gamma)\eta_\theta}+\frac{(2(E-1)+1/2)\beta\eta_\theta V^2(N+1)}{1-\gamma}\left(\frac{N}{(1-\gamma)^2\xi^2}+1\right)+\frac{2N}{T\eta_\lambda(1-\gamma)^2\xi^2}+\frac{2N\eta_\lambda}{(1-\gamma)^2}\\
&\quad+\left(\frac{1}{1-\gamma}+\frac{2}{(1-\gamma)^2\xi}\right)\sqrt{\frac{1}{1-\gamma}\left\|\frac{\nu^*}{\nu_0}\right\|_\infty\left(\epsilon_{bias}+\frac{2\left(2\sqrt{d}VL_\pi+\frac{2\sqrt{d}}{1-\gamma}+VL_\pi\right)^2}{K}\right)}\\
&\quad +\frac{|\gA|L_\pi (E-1)\eta_\theta V}{(1-\gamma)^2}\left(\frac{N}{(1-\gamma)\xi}+1\right)^2\Bigg\}.\\
\end{aligned}
$$
\end{proof}
\begin{proof}[Proof of Theorem~\ref{Theorem_main}]
    Theorem~\ref{Theorem_main} is a direct consequence of the combination of Lemma~\ref{Lemma_bound_of_reward} and Lemma~\ref{Lemma_bound_on_constraint}.
\end{proof}

\end{document}